\documentclass[12pt]{amsart}

\usepackage[margin=1in]{geometry}
\usepackage{graphicx} 
\usepackage[all]{xy} 
\usepackage{schemabloc}

\usepackage{amssymb}
\usepackage{amsthm}
\usepackage{amsmath}
\usepackage{latexsym}
\usepackage[mathscr]{euscript}
\usepackage{pb-diagram}
\usepackage{latexsym}
\usepackage{epsfig}
\usepackage{graphics}
\usepackage{mathabx}

\newcommand\sep{\mathrm{sep}}
\newcommand\diam{\mathrm{diam}}

\newcommand{\beq}[1]{\begin{equation}\label{#1}}
\newcommand{\eeq}{\end{equation}}

\def\compcirc {\mbox{\hspace{.05cm}}\raisebox{.04cm}{\tiny  {$\circ$ }}}

\newcommand\R{\mathbb{R}}
\newcommand\N{\mathbb{N}}

\newcommand\dmtwo[1]{\left(\begin{smallmatrix}
			0 & {#1}\\
     			{#1} & 0
                     \end{smallmatrix}\right)}

\newcommand\dmthree[3]{\left(\begin{smallmatrix}
			0 & {#1}&{#2}\\
     			{#1} & 0 & {#3}\\
			{#2} &  {#3}& 0
                     \end{smallmatrix}\right)}

\newcommand \cfunc{{\mathfrak C}}
\newcommand \hfunc{{\mathfrak H}}
\newcommand \rfunc{{\mathfrak R}}
\newcommand \func[1]{{\mathfrak #1}}

\newcommand \tfunc{{\mathfrak T}}
\newcommand \ifunc{{\mathfrak I}}

\newcommand\ob{\mathrm{ob}}
\newcommand\mor{\mathrm{Mor}}
\newcommand\im{\mathrm{Im}}

\newcommand \miso{\underline{\mathcal M}^{iso}}
\newcommand \mmon{\underline{\mathcal M}^{inj}}
\newcommand \mgen{\underline{\mathcal M}^{gen}}
\newcommand \many{\underline{\mathcal M}}
\newcommand \catc{\underline{\mathcal C}}

\newcommand \catp{\underline{\mathcal P}}

\newcommand \cat[1]{\underline{\mathcal {#1}}}

\newcommand \catr{\underline{\mathbb E}}

\newtheorem{corollary}{Corollary}[section]
\newtheorem{theorem}{Theorem}[section]
\newtheorem{definition}{Definition}[section]
\newtheorem{example}{Example}[section]
\newtheorem{lemma}{Lemma}[section]

\newtheorem{claim}{Claim}[section]
\newtheorem{remark}{Remark}[section]

\begin{document} 

\title{Classifying clustering schemes}

\author{Gunnar Carlsson   and   Facundo M\'emoli}
\address{Department of Mathematics, Stanford University,
Stanford, CA 94305}
\thanks{This work is supported by DARPA grant HR0011-05-1-0007 and  ONR grant N00014-09-1-0783.}
\email{$\{$gunnar,memoli$\}$@math.stanford.edu}
\date{November 2010}

\maketitle

\begin{abstract} 
Many clustering schemes are defined by optimizing an objective
function defined on the partitions of the underlying set of a finite metric space.
 In this paper, we construct a framework for studying what happens when we instead impose various
structural conditions on the clustering schemes, under the general heading of \emph{functoriality}. 

Functoriality refers to the idea that one
should be able to compare the results of clustering algorithms as one
varies the data set, for example by adding points or by applying
functions to it.  We show that within this framework, one can prove a
theorems analogous to one of J. Kleinberg \cite{kleinberg}, in which for example
one obtains an existence and uniqueness theorem instead of a
non-existence result.

We obtain a full
classification of all clustering schemes satisfying a condition we refer
to as excisiveness.  The classification can be changed by varying the
notion of maps of finite metric spaces. The conditions occur naturally
when one considers clustering as the statistical version of the geometric
notion of connected components.  By varying the degree of functoriality that one requires from the schemes it is possible to construct richer families of clustering schemes that exhibit sensitivity to density.
\end{abstract}

\section{Introduction}
Clustering techniques play a very central role in various parts of
data analysis.  This type of methods  can give important clues to the structure of data
sets, and therefore suggest results and hypotheses in the underlying
science.  There are many interesting methods of clustering available,
which have been applied to good effect in dealing with many datasets
of interest, and are regarded as important methods in exploratory
data analysis.

Most methods begin with a data set equipped with a notion of distance or metric.  Starting with this input, a method might select a partition of the data set as one which optimizes a choice of an objective function.  Other methods such as single, average, and complete linkage define clusterings, or rather nested families of clusterings, using a \emph{linkage function} defined between clusters.

Desirable properties of clustering algorithms come from practitioners who have intuitive notions of what is a good clustering: ``they know it when they see it''. This is of course not satisfactory and a theoretical understanding needs to be developed. However, one thing this intuition reflects is the fact that density needs to be incorporated in the clustering procedures. Single linkage clustering, a procedure that enjoys several nice theoretical properties, is notorious for its insensitivity to density, which is manifested in the so called \emph{chaining effect}  \cite{lance67general}.  

Other methods such as average linkage, complete linkage \cite[Chapter 3]{clusteringref} and $k$-means share the property that they exhibit some sort of sensitivity to density, but are unstable in a sense which has been made theoretically precise \cite{clust-um} and are therefore not well supported by theory. We believe that this disconnect between theory and practice should not exist, and in particular, in \cite{multi-param} we have constructed a theoretically sound framework that incorporates density via the use of $2$-dimensional persistence ideas. 

 Given that there is relatively little theory surrounding clustering methods, the subject is regarded by many as consisting of a collection of ad hoc methods for which it is difficult to generate a coherent picture \cite{science_art}.  One exception to this rule is a paper of J. Kleinberg \cite{kleinberg}, in which he proves a theorem which demonstrates the impossibility of constructing a clustering algorithm with three seemingly plausible properties. Kleinberg regards a clustering scheme as a map $\cfunc$ that assigns each finite metric space $(X,d_X)$ one of its partitions. His ``axioms'' are \begin{itemize}
\item \emph{Scale invariance}: $\cfunc(X,d_X)=\cfunc(X,\lambda\cdot d_X)$ for all $\lambda>0$;
\item  \emph{Surjectivity}: for all partitions $P$ of $X$ there exists a metric $d_X$ on $X$ with $\cfunc(X,d_X)=P$; and
\item  \emph{Consistency}: upon reducing distances between points in the same cluster of $X$ (as produced by $\cfunc$), and increasing distances between points in different clusters, the result of applying the method to the new metric does not change.
\end{itemize}

\begin{theorem}[\cite{kleinberg}]
There exists no clustering algorithm that simultaneously satisfies scale invariance, surjectivity and consistency.
\end{theorem}

Kleinberg also points out that his results shed light on the trade-offs one has
to make in choosing clustering algorithms.  In this paper, we produce
a variation on this theme, which we believe also has implications for
how one thinks about and applies clustering algorithms.

 In our earlier paper \cite{clust-um}, in the context of hierarchical methods, we have established a variant on Kleinberg's theorem, in which instead of impossibility, one actually proves an \emph{existence and uniqueness} result, again assuming three plausible properties, slightly different from Kleinberg's.  In the present paper, we make one more step towards understanding the theoretical foundations of clustering methods by expanding on the theme of \cite{clust-um} and developing a context in which one can obtain much richer classification results, which include versions of a number of familiar methods, including analogues of the clique clustering methods familiar in network and graph problems.  We now describe our context and results.

\subsection{Overview of our methods and results}
We regard clustering as the ``statistical version" of the geometric notion of constructing the connected components of a topological space.  Recall (see, e.g. \cite{munkres}) that the path
components of a topological space $X$ are the equivalence classes of
points in the space under the equivalence relation $\sim _{path}$,
where, for $x,x' \in X$, we have $x \sim _{path} y$ if and only if
there is a continuous map $\varphi: [0,1] \rightarrow X$ so that
$\varphi (0) = x$ and $\varphi (1) = x'$.  In other words, two points
in $X$ are in the same path component if they are connected by a
continuous path in $X$.   We let $\pi _0 (X)$ denote the set of connected components of the space $X$.  A crucial observation about $\pi _0$ is that it is a \textbf{functor}, i.e. that not only does one construct a set (of components) associated to every topological space, but additionally, to every continuous map $f : X \rightarrow Y$ one associates a map of sets $\pi _0(X) \rightarrow \pi _0 (Y)$.  We take the point of view that this \textbf{functoriality} is a key property to export from topological spaces to the finite metric spaces which form the input to clustering algorithms.  Functoriality has been a very useful framework
for the discussion of a variety of problems within mathematics over
the last few decades, see
\cite{maclane}. 

\begin{remark}

 The reader may not be familiar with the notion of functoriality, as presented in this paper.   The fundamental idea behind functoriality is that it is crucial to study the maps or transformations between mathematical objects as well as the objects themselves.  Interpreted at this level of generality, one can readily see the importance of this notion via a number of examples.  

\begin{enumerate}
\item{{\bf Fourier analysis:}  The key observation of Fourier analysis in its various forms is the the study of the symmetry groups of  a space allows one to construct very useful and conceptually important bases for spaces of functions on the space.  So, the standard Fourier basis for complex valued functions on the circle is constructed by selecting functions which satisfy certain transformation laws under the rotational symmetries.  This is of course also true for, e.g., the spherical harmonic basis for the functions on $\Bbb{R}^3$.  Functoriality can be regarded as a transformation law.  }
\item{{\bf Normal forms for matrices:}  The analysis of the conjugation action of groups of invertible matrices on themselves gives rise to normal forms of matrices, such as Jordan normal form.  This means that the conjugation symmetries on spaces of matrices allows one to understand their organization}
\item{{\bf Algebraic topology:}  In algebraic topology, one associates to topological spaces algebraic objects called {\em homology groups}. It is critical for applications of these groups, and for computing them, that they obey a transformation law for any continuous map of spaces.  This transformation law turns out to be very powerful, in that it is a key factor in making a small axiomatic description of the homology groups.  Just as transformation laws for symmetries determine the Fourier basis, so the functoriality determines the homology groups.   }
\item{{\bf Galois theory:} In Galois theory, groups of symmetries allow one to describe the sets of solutions of algebraic equations over fields, such as the rational numbers.  The observation that equations were associated to symmetry groups revolutionized number theory, and to this day the study of various kinds of transformation laws in modular forms continues to have a  powerful impact on the field.   }
\end{enumerate} 

\end{remark}
The point of the present paper is to illustrate that functoriality allows a very useful framework for classifying large families of clustering schemes.  This framework will include clustering schemes which are sensitive to effects by density within the domain data set.  Moreover, the functoriality of the construction of path connected components is a key feature in all aspects of topology, and the functoriality of clustering constructions should play a similar role in the computational topology of finite data sets.  

Then, we require that a standard\footnote{As opposed to hierarchical.} clustering algorithm $\cfunc$ be a rule which assigns to every finite metric space $X$ a set of clusters $\cfunc (X)$, subject to additional the requirement that every map $f: X \rightarrow Y$ induce a map of sets $\cfunc(f) : \cfunc (X) \rightarrow \cfunc (Y)$.  

The question is what the replacement for continuous maps (i.e. \emph{morphisms} of topological spaces) should be, and we find that there are many choices.  We consider three nested classes of maps, namely isometries (called $iso$), distance non-increasing maps (called $gen$, for general), and distance non-increasing maps which are injections on the sets of points (called $inj$), with the hierarchy being given by $iso\subset inj\subset gen.$ We find that functoriality with respect to $gen$ is quite restrictive, in that it singles out single linkage clustering as the only functor which satisfies the $gen$ functoriality requirement.  On the other hand, functoriality with respect to $iso$ appears not to be restrictive enough, in that it permits the specification of an arbitrary clustering on every isometry class of finite metric spaces.

One of the features of our classification is the proof that for the $gen$ and $inj$ cases, a certain implicitly defined  property of clustering schemes called \textbf{excisiveness}, which amounts to requiring that the clustering method is idempotent on each of the blocks of the partition it generates, is \emph{equivalent} to the existence of a certain kind of \emph{explicit generative model} for the scheme.

 Mathematical results which prove the existence of a generative model given that some externally defined properties are satisfied are very powerful in mathematics.  For example, the result that a system satisfying the axioms for a vector space over a field $\mathbb{F}$ admits a basis is very useful, as is  the fact that a graph with no odd order cycles is bipartite. 

The equivalence between excisiveness and the existence of a generative model for the clustering scheme has the interpretation that excisive schemes are parametrized by sets of ``test metric spaces", which are those connected into a single cluster by the scheme.  The method associated to such a family is then defined using a criterion which asserts that two points $x$ and $x'$ in a metric space $X$ belong to the same cluster if there is a sequence of points $\{ x_0, x_1, \ldots , x_n \}$ with $x_0 = x$, $x_n = x'$, and so that for each pair $x_i, x_{i+1}$ there is a morphism in $gen$ or $inj$ from one of the test metric spaces into $X$ with $x_i$ and $x_{i+1}$ in its image. 

 A key feature of the choice of morphisms in $inj$ is that (since they are injective) this permits one to take account of \emph{density} in a certain sense.  If one considers functoriality with respect to $gen$, one finds that the morphisms might collapse a large dense collection of points into a single point, which makes it impossible to take density into account in building the clusters.  Our classification of functors based on $inj$ includes analogues of ``clique clustering" \cite{hungarian-clique} and the DBSCAN algorithm \cite{dbscan}, and such methods clearly have the effect of taking a version of density into account. 

Whereas our classification of $gen$ functorial schemes yields that single linkage is the unique possible scheme, which turns out to be excisive, we find large classes of non-excisive clustering schemes in $inj$. Our construction of such families relies on a study of the metric invariants that are well behaved under $inj$ morphisms.  Finally, it turns out that all schemes in $inj$ that admit a \emph{finite} generative model arise as the composition of single linkage clustering with an application that changes the underlying metric of the input metric space. This result has clear practical consequences which we discuss.

We believe that functoriality, incorporating as it does the idea that clustering should not operate only on finite metric spaces as isolated objects, but rather on the metric space together with other metric spaces which are related to it via morphisms, is a very desirable property.  We can point to its success within topology as an extremely powerful tool which pins down many constructions as the unique solution to existence problems, and which is the key to most applications 
of topological methods.   It is also true that functoriality is necessary in order to extend the methods of clustering to higher order connectivity information with computational topology of finite data sets \cite{carlsson-09}.  

In this paper we demonstrate that the use of restricted notions of functoriality allows us to study clustering schemes which take proxies for density into account.  This is important because it is clear that some of the value of alternate schemes such as average and complete linkage (which are not functorial)  is due precisely to their ability to do this.  Therefore, some of the methods we study will not suffer from the problems inherent in single linkage clustering, such as chaining  \cite[Section 7.4]{jardine-sibson} \cite{multi-param}.

\subsection{Organization of the paper}  In Section
\ref{sec:prelim} we discuss basic background concepts that we use in our presentation. Section \ref{cats}  introduces the concept of a category and gives many examples and constructions we use in the paper. The concept of a functor and functoriality are discussed in Section \ref{sec:functors}.  
We present our main characterization results for standard clustering methods in Section
\ref{results-sc}. We study hierarchical methods in Section \ref{results-hc}. Concluding remarks are given in Section \ref{various}.

\section{Preliminaries}
\label{sec:prelim}
For a finite set $X$ we denote by $\mathcal{P}(X)$ the set of all partitions of $X$. For each $k\in\N$, $\mathcal{M}_k$ denotes the collection of all finite metric spaces with cardinality $k$ and $\mathcal{M}=\bigcup_{k\in\N} \mathcal{M}_k$ denotes the collection of all finite metric spaces. For $(X,d_X)\in\mathcal{M}$ let $\sep(X,d_X):=\min_{x\neq x'}d_X(x,x')$ be the \emph{separation} of $X$, $\diam(X,d_X):=\max_{x,x'\in X}d_X(x,x')$  be the \emph{diameter} of $X$, and for $\lambda>0$ we denote by $\lambda\cdot X$ the metric space  $(X,\lambda\cdot d_X)$. For $\delta>0$ and $k\geq 2$ we let $\Delta_k(\delta)$ denote the metric space with $k$ points whose inter-point distances equal $\delta$.

\begin{definition}\label{def-equiv-rips} On $(X,d_X)\in\mathcal{M}$, for each $\delta\geq 0$, we define the equivalence relation $\sim _\delta$, where $x \sim _\delta x^{\prime}$ if and only if there is a sequence $x_0, x_1, \ldots , x_k \in X$ so that $x_0 = x, x_k = x^{\prime}$, and $d_X(x_i, x_{i+1}) \leq \delta$ for all $i$.  We denote by $[x]_\delta$ the equivalence class of $x\in X$ under $\sim_\delta$.
\end{definition}

The following lemma follows directly from the definition of $\sim_\delta$.
\begin{lemma}\label{prop:props-uaster} Let $(X,d_X)\in\mathcal{M}$ and $\delta\geq 0$. Then, 
$[x]_\delta\neq [x']_\delta$ for some $x,x'\in X$ if and only if { $$d_X(a,a')>\delta,\,\,\mbox{for all $a\in [x]_\delta$ and $a'\in [x']_\delta$}.$$}
\end{lemma}

Let $X$ be any finite set and $W_X:X\times X\rightarrow \R_+$ be a symmetric function s.t. $W_X(x,x)=0$ for all $x\in X$. Then, the \emph{maximal sub-dominant ultrametric} relative to $W_X$ is given by
{\beq{eq:U}\mathcal{U}(W_X)(x,x')=\min\Big\{\max_{i}W_X(x_i,x_{i+1}),\,\{x_i\}_{i=0}^k\,\mbox{s.t. $x_0=x$ and $x_k=x'$}\Big\}
.\eeq}
This construction and its correctness are standard \cite[\S 6.4.3]{handbook-um}.

We define objects which will encode
the notion of ``multiscale'' or ``multiresolution'' partitions.

\begin{definition}\label{def:persistent-sets}
A {\em persistent set} is a pair $(X, \theta_X)$, where $X$ is a finite
set, and $\theta_X$ is a function from the non-negative real line $[0,+
\infty)$ to ${\mathcal P}(X)$  so that the following
properties hold.
\begin{enumerate}
\item{ If $r \leq s$, then $\theta_X (r)$ refines $\theta_X(s)$.  }
\item{For any $r$, there is a number $\epsilon > 0$ so that $\theta_X (r^{\prime}) = \theta_X (r)$ for all $r^\prime\in[r,r+\epsilon]$.} 
\end{enumerate}
If in addition there exists $t>0$ s.t. $\theta_X(t)$ consists of the
single block partition for all $r\geq t$, then we say that
$(X,\theta_X)$ is a \emph{dendrogram}.\footnote{In the paper we will be
using the word dendrogram to refer both to the object defined here and
to the standard graphical representation of them.}
\end{definition}

\begin{example}\label{rips}  Let $(X,d_X)$ be a finite metric space.  Then we can associate to $(X,d_X)$ the persistent set whose underlying set is $X$, and where for each $r\geq 0$, blocks of  the partition $\theta_X (r)$ consist of the  equivalence classes under the equivalence relation $\sim _r$.
\end{example}

\begin{example} \label{hierarchical} Here we consider the family
    of Agglomerative Hierarchical clustering techniques,
    \cite{clusteringref}. We (re)define these by the recursive
    procedure described next. Let $X=\{x_1,\ldots,x_n\}$ and let
    $\mathcal L$ denote a family of \emph{linkage functions},
    i.e. functions which one uses for defining the distance between
    two clusters. Fix $l\in {\mathcal L}$. For each $R>0$ consider the
    equivalence relation $\sim_{l,R}$ on blocks of a partition
    $\Pi\in{\mathcal P}(X)$, given by \mbox{${\mathcal B}\sim_{l,R}
      {\mathcal B}'$} if and only if there is a sequence of blocks
    ${\mathcal B}={\mathcal B}_1,\ldots,{\mathcal B}_s={\mathcal B}'$
    in $\Pi$ with $l({\mathcal B}_k,{\mathcal B}_{k+1})\leq R$ for
    $k=1,\ldots,s-1$. Consider the sequences $r_1,r_2,\ldots\in
    [0,\infty)$ and $\Theta_1,\Theta_2,\ldots\in \mathcal{P}(X)$ given
    by $\Theta_1:=\{x_1,\ldots,x_n\}$ and for $i\geq 1$,
    $\Theta_{i+1}=\Theta_i/\sim_{l,r_i}$ where $r_i:=
    \min\{l({\mathcal B},{\mathcal B}'),\,{\mathcal B},{\mathcal
      B}'\in\Theta_i,\,{\mathcal B}\neq {\mathcal B}'\}.$ Finally, we
    define $\theta^l:[0,\infty)\rightarrow \mathcal{P}(X)$ by
    $r\mapsto \theta^l(r):=\Theta_{i(r)}$ where $i(r):=\max\{i|r_i\leq
    r\}$. 

    Standard choices for $l$ are \emph{single linkage}:
    $$l({\mathcal B},{\mathcal B}')=\min_{x\in {\mathcal B}}\min_{x'\in
      {\mathcal B}'}d_X(x,x');$$ \emph{complete linkage} $$l({\mathcal
      B},{\mathcal B}')=\max_{x\in {\mathcal B}}\max_{x'\in {\mathcal
        B}'}d_X(x,x');$$ and \emph{average linkage}: $$l({\mathcal
      B},{\mathcal B}')=\frac{\sum_{x\in {\mathcal B}}\sum_{x'\in
        {\mathcal B}'}d_X(x,x')}{|{\mathcal B}| \cdot |{\mathcal B}'|}.$$
    It is easily verified that the notion discussed in Example
    \ref{rips} is \emph{equivalent} to $\theta^l$ when $l$ is the
    single linkage function. Note that, unlike the usual definition of
    agglomerative hierarchical clustering, at each step of the
    inductive definition we allow for more than two clusters to be
    merged.
\end{example}



\section{Categories}\label{cats}

\subsection{Definitions and Examples}
In this section, we will give a brief description of the theory of
categories and functors, which will be the framework in which we state
the constraints that we require of our clustering algorithms.  An
excellent reference for these ideas is \cite{maclane}. 

Categories are useful mathematical constructs that encode the nature
of certain objects of interest together with a set of
admissible maps between them. This formalism is
extremely useful for studying classes of mathematical objects which
share a common structure, such as sets, groups, vector spaces, or
topological spaces.  The definition is as follows.

\begin{definition} A \textbf{category} $\underline{C}$ consists of: 
\begin{itemize}

	\item{ A collection of \textbf{objects} $\ob(\underline{C})$ (e.g.
	sets, groups, vector spaces, etc.)}

	\item{ For each pair of objects $X,Y \in \ob(\underline{C})$, a
	set\\
	 $\mor_{\underline{C}} (X,Y)$, the \textbf{morphisms} from $X$ to $Y$
	(e.g. maps of sets from $X$ to $Y$, homomorphisms of groups
	from $X$ to $Y$, linear transformations from $X$ to $Y$,
	etc. respectively)}
	
	\item{ Composition operations:\\ $\compcirc :
	\mor_{\underline{C}} (X,Y) \times \mor_{\underline{C}} (Y,Z)
	\rightarrow \mor_{\underline{C}} (X,Z)$, corresponding to
	\textbf{composition} of set maps, group homomorphisms, linear
	transformations, etc.  }
	
	\item{For each object $X \in \underline{C}$, a distinguished element $\textrm{id}_X\in \mor_{\underline{C}}(X,X)$, called the \textbf{identity} morphism.}
\end{itemize}
The composition is assumed to be associative in the obvious sense, and for any $f \in \mor_{\underline{C}} (X,Y)$, it is assumed that $\textrm{id}_Y \compcirc f = f$ and $f \compcirc \textrm{id}_X = f$. 
\end{definition} 

\begin{example}[The category of sets]
Denote by $\underline{Sets}$ the category whose objects are all sets, and the morphisms between two sets $A$ and $B$ are all the functions from $A$ to $B$.
\end{example}

A morphism $f:X \rightarrow Y$ which has a two sided inverse $g:
Y\rightarrow X$, so that $f \compcirc g = \textrm{id}_Y$ and $g \compcirc f =
\textrm{id}_X$, is called an {\bf isomorphism}.  Two objects which are
isomorphic are intuitively thought of as ``structurally
indistinguishable" in the sense that they are identical except for
naming or choice of coordinates.  For example, in the category of
sets, the sets $\{ 1,2,3 \}$ and $\{ A,B,C \}$ are isomorphic, since
they are identical except for the choice made in labelling the elements.


\begin{example}
Here we consider four very simple categories, that we are going to refer to as $\underline{0},\underline{1},\underline{2}$ and $\underline{3}$ respectively. 

\begin{itemize}

\item The category $\underline{0}$ has $ob(\underline{0})=\emptyset$ and all the conditions in the definition above are trivially satisfied.

\item Consider the category $\underline{1}$  with exactly one object  $A$ and one morphism:  $\mor_{\underline{1}}(A,A)=f$. It follows that $f$ must be the identity morphism $id_A$. This is represented graphically as follows:

$$\xymatrix{ 
A \ar@(ur,dr)[]|{id_A}}$$

\item The category $\underline{2}$ has exactly two objects $A$ and $B$ and three morphisms: the identities from $A$ to $A$ and from $B$ to $B$ and exactly one morphism in $\mor_{\underline{2}}(A,B)$:

$$\xymatrix{ 
A \ar@(ul,dl)[]|{id_A} \ar[rr]|f 
&& B  \ar@(ur,dr)[]|{id_B} }$$

\item Finally, the category $\underline{3}$ has exactly three objects $A$, $B$ and $C$ and six morphisms: the identities from $A$ to $A$,  from $B$ to $B$ and $C$ to $C$, and three more morphisms,  $\mor_{\underline{3}}(A,B)=f$, $\mor_{\underline{3}}(B,C)=g$ and $\mor_{\underline{3}}(A,C)=h$:

$$\xymatrix{ 
A \ar@(ul,dl)[]|{id_A} \ar[rr]|f   \ar[rd]|h && B  \ar@(ur,dr)[]|{id_B} \ar[ld]|g\\
 & C \ar@(dl,dr)[]|{id_C}
}
$$

\noindent
Now, note that in order to satisfy the composition rule one must have $h=g\circ f$.

\end{itemize}
\end{example}

\begin{example}
A more surprising example of a category is given by the following construction. A \emph{monoid} $(M,\ast,e)$ consists of a set $M$ equipped with a binary operation $\ast :M\times M\rightarrow M$ s.t. $(A\ast B)\ast C = A\ast (B\ast C)$ for all $A,B,C\in M$ and an element $e\in M$ s.t. $e\ast A = A\ast e =A$ for all $A\in M$. A monoid can be made into a category $\underline{M}$ with a single object $O$ and $\mor_{\underline{M}}(O,O)=M$ and $id_O$ being represented by $e$. The composition of morphisms is induced by the operation $\ast$.
\end{example}

\begin{example}
Let $\mathbb{E}$ denote the extended positive real line $[0,+\infty]$. We construct the category $\catr$ with $\ob(\catr)=\mathbb{E}$ and $\mor_{\catr}(a,b) = \{(a,b)\}$  if $a\geq b$, and $\mor_{\catr}(a,b)=\emptyset$. The composition is given by $(a,b)\compcirc (b,c) = (a,c)$.
\end{example}

Now we discuss the core  constructions for this paper. 
\begin{definition}[\bf $\catc$, a category of outputs of standard clustering schemes]{
Let $Y$ be a finite set, $P_Y\in \mathcal{P}(Y)$, and $f:X\rightarrow Y$ be a set map. We define $f^*(P_Y)$ to be the partition of $X$ whose blocks are the sets $f^{-1}({B})$ where ${B}$ ranges over the blocks of $P_Y$.
 We construct the category $\catc$ of outputs of standard
  clustering algorithms with $\ob(\catc)$ equal to all
  possible pairs $(X,P_X)$  where $X$ is a finite set and $P_X$ is a partition of $X$: $P_X\in {\mathcal P}(X)$. For objects $(X,P_X)$ and $(Y,P_Y)$ one sets $\mor_{\catc}\big((X,P_X),(Y,P_Y)\big)$ to be the set of all maps $f:X\rightarrow Y$ with the property that $P_X$ is a refinement of $f^*(P_Y)$.
  }
\end{definition}

\begin{example}
Let $X$ be any finite set, $Y=\{a,b\}$ a set with two elements, and $P_X$ a partition of $X$. Assume first that $P_Y=\{\{a\},\{b\}\}$ and let $f:X\rightarrow Y$ be any map. Then, in order for $f$ to be a morphism in $\mor_{\catc}\big((X,P_X),(Y,P_Y)\big)$ it is necessary that $x$ and $x'$ be in different blocks of $P_X$  whenever $f(x)\neq f(x')$. Assume now that $P_Y = \{a,b\}$ and $g:Y\rightarrow X$. Then, the condition that $g\in\mor_{\catc}\big((Y,P_Y),(X,P_X)\big)$ requires that $g(a)$ and $g(b)$ be in the same block of $P_X$.
\end{example}

We will also construct a category of \emph{persistent sets}, which will constitute the output of hierarchical clustering functors. 

\begin{definition}[\bf $\catp$, a category of outputs of hierarchical clustering schemes]\label{definition:P}
{ Let $(X, \theta_X), (Y, \theta_Y)$ be persistent sets.  A
map of sets $f: X \rightarrow Y$ is said to be {\em persistence
preserving} if for each $r \in \mathbb{R}$, we have that $\theta_X (r)$
is a refinement of $ f^*(\theta_Y(r))$.  We define a category $\underline{{\mathcal P}}$
whose objects are persistent sets, and where $\mor_{\underline{{\mathcal
P}}}((X,\theta_X), (Y, \theta_Y))$ consists of the set maps from $X$ to $Y$
which are persistence preserving.  }
\end{definition}
It is easily verified that the
composite of persistence preserving maps is persistence preserving,
and that any identity map is persistence preserving.  A simple example is shown in Figure
\ref{fig:pers-preserv}.
\begin{figure}[htb]
 	\centering
 	\includegraphics[width=1\linewidth]{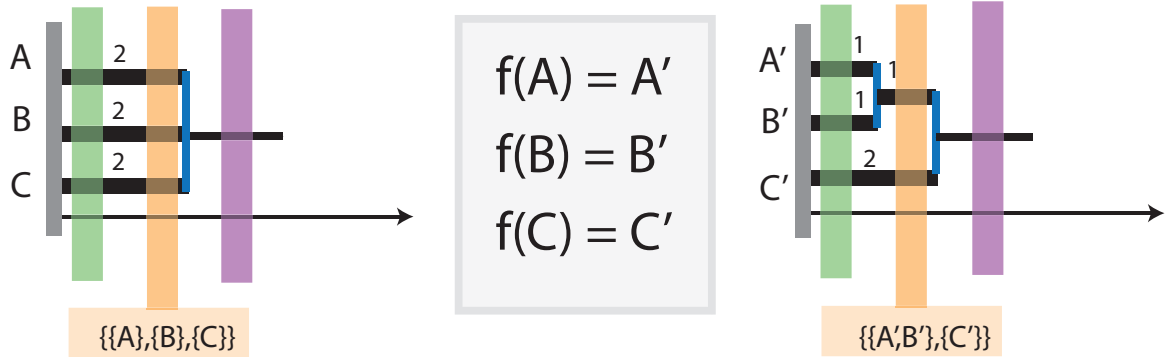}
 	\caption{Two persistent sets $(X,\theta_X)$ and $(Y,\theta_Y)$
 	represented by their dendrograms. On the left one defined in
 	the set $X = \{A,B,C\}$ and on the right one defined on the
 	set $Y=\{A',B',C'\}$. Consider the given set map
 	$f:X\rightarrow Y$. Then we see that $f$ is persistence
 	preserving since for each $r\geq 0$, the partition $\theta_X(r)$
 	is a refinement of $f^*(\theta_Y(r))$. Indeed, there are three
 	interesting ranges of values of $r$. Pick for example $r$ like
 	in the orange shaded area: $r\in[1,2)$. Then $\theta_Y(r) =
 	\{\{A',B'\},\{C'\}\}$ and hence $f^*(\theta_Y(r)) =
 	\{f^{-1}(\{A',B'\}),\{f^{-1}(C')\}\} = \{\{A,B\},\{C\}\}$
 	which is indeed refined by
 	$\theta_X(r)=\{\{A\},\{B\},\{C\}\}$. One proceeds similarly for
 	the other two cases.}  \label{fig:pers-preserv}
 \end{figure}

\subsection{Three categories of finite metric spaces.}
{ We will describe three categories $\miso$, $\mmon$, and $\mgen$, whose
collections of objects will all consist of the collection of finite
metric spaces $\mathcal{M}$.  For $(X,d_X)$ and $(Y, d_Y)$ in $\mathcal{M}$, a  map $f: X \rightarrow Y$ is said to be {\bf distance
non increasing} if for all $x,x^{\prime} \in X$, we have
$d_Y(f(x),f(x^{\prime})) \leq d_X(x,x^{\prime})$.  It is easy to check
that composition of distance non-increasing maps are also distance
non-increasing, and it is also clear that $\textrm{id}_X$ is always distance
non-increasing.  We therefore have the category $\mgen$, whose objects are finite metric spaces, and so that for
any objects $X$ and $Y$, $\mor_{\mgen}(X,Y)$ is
the set of distance non-increasing maps from $X$ to $Y$, cf.
\cite{isbell} for another use of this class of maps. It is clear that compositions
of injective maps are injective, and that all identity maps are injective, so we
have the new category $\mmon$, in which
$\mor_{\mmon}(X,Y)$ consists of the \textbf{injective distance
non-increasing maps}.  Finally, if $(X,d_X)$ and $(Y,d_Y)$ are finite
metric spaces, $f:X
\rightarrow Y$ is an {\bf isometry} if $f$ is bijective and $d_Y(f(x),
f(x^{\prime})) = d_X(x,x^{\prime})$ for all $x$ and $x^{\prime}$.  It
is clear that as above, one can form a category $\miso$ whose objects are finite metric spaces and whose morphisms
are the isometries.   Furthermore, one has inclusions 
\begin{equation}\label{eq:inclu-cats}
\miso\subseteq \mmon
\subseteq \mgen
\end{equation}
 of subcategories (defined as in \cite{maclane}).  Note that although
 the inclusions are bijections on object sets, they are proper
 inclusions on morphism sets, i.e. in general they are not surjective.
 }

\begin{remark}\label{remark:mgen}
The category $\mgen$ is special in that for any pair of finite metric spaces $X$ and $Y$, \mbox{$\mor_{\mgen}(X,Y)\neq \emptyset$}. Indeed, pick $y_0\in Y$ and define $\phi:X\rightarrow Y$ by $x\mapsto y_0$ for all $x\in X$. Clearly, $\phi\in\mor_{\mgen}(X,Y)$. This is not the case  for $\mmon$ since in order for $\mor_{\mmon}(X,Y)\neq \emptyset$ to hold it is necessary (but not sufficient in general) that $|Y|\geq |X|$.
\end{remark}

\section{Functors and functoriality.} \label{sec:functors}
Next we introduce the key concept in our discussion, that of a
 \emph{functor}.  We give the formal definition first, and several examples will appear as different constructions that we use in the paper.
\begin{definition}[\bf Functor]\label{def:functor}
Let $\underline{C}$ and $\underline{D}$ be categories.  Then a {\em functor} from $\underline{C}$ to $\underline{D}$ consists of:

\begin{itemize}

	\item{A map of sets $F:\ob(\underline{C}) \rightarrow
	\ob(\underline{D})$.} 
	
	\item{For every pair of objects $X,Y \in \underline{C}$ a map
	of sets $\Phi(X,Y):\mor_{\underline{C}}(X,Y) \rightarrow
	\mor_{\underline{D}}(FX,FY)$ so that
	\begin{enumerate}
		\item{ $\Phi(X,X)(\textrm{id}_X) = \textrm{id}_{F(X)}$ for all $X \in
		\ob(\underline{C})$, and}

            \item{$ \Phi (X,Z)(g \compcirc f) = \Phi (Y,Z)(g)\compcirc \Phi (X,Y)(f)$ for all $f \in \mor_{\underline{C}}(X,Y)$ and $g \in \mor_{\underline{C}}(Y,Z)$.}
	\end{enumerate} }
\end{itemize}
Given a category $\underline{C}$,  an \textbf{endofunctor} on $\underline{C}$ is any functor $F:\underline{C}\rightarrow \underline{C}.$
\end{definition}
\begin{remark}
In the interest of clarity, we will always refer to the pair $(F,\Phi)$ with a
single letter $F$. See diagram (\ref{eq:diagram-general}) below for an example.
\end{remark}


\begin{example}{\bf (Forgetful functors)}  When one has two categories $\underline{C}$ and $\underline{D}$, where the objects in $\underline{C}$ are objects in $\underline{D}$  equipped with some additional structure and the morphisms in $\underline{C}$ are simply the morphisms in $\underline{D}$ which preserve that structure, then we obtain the ``forgetful functor" from $\underline{C}$ to $\underline{D}$, which carries the object in $\underline{C}$ to the same object in $\underline{C}$, but regarded without the additional structure.  For example, a group can be regarded as a set with the additional structure of multiplication and inverse maps, and the group homomorphisms are simply the set maps which respect that structure.  Accordingly, we have the functor from the category of groups to the category of sets which ``forgets the multiplication and  inverse".  Similarly, we have the forgetful functor from $\catc$ to the category of sets, which forgets the  presence of $P_X$ in the output set $(X, P_X)$.  
\end{example}

\begin{example}   
The inclusions $ \miso\subseteq \mmon \subseteq \mgen $ are both functors.  
\end{example}

\begin{example}[{Scaling functor}]\label{scaling} For any  $\lambda >0$ we define an endofunctor $\sigma _{\lambda} : \mgen
\rightarrow \mgen$ on objects by 
$  \sigma _{\lambda}(X,d_X) = (X,\lambda \cdot d_X)
$
and on morphisms by $\sigma _{\lambda} (f) = f$.  One easily verifies that if $f$ satisfies the conditions for being a morphism in $\mgen$ from $(X,d_X)$ to $(Y,d_Y)$, then it readily satisfies the conditions of being a morphism from $(X,\lambda \cdot d_X)$ to $(Y, \lambda \cdot d_Y)$.  Clearly, $\sigma_\lambda$ can also be regarded as an endofunctor in $\miso$ and $\mmon.$

Similarly, we define a functor $s_{\lambda} : \underline{{\mathcal P}} \rightarrow \underline{{\mathcal P}}$ by setting $s_{\lambda}(X, \theta _X) = (X, \theta_X ^{\lambda})$, where 
$\theta_X ^{\lambda}(r) = \theta_X (\frac{r}{\lambda})$. 
\end{example}

\subsection{Clustering algorithms as functors}

The notion of categories, functors and functoriality provide
  useful framework for studying algorithms.  One first defines a class of
  input objects $\underline{\mathcal{I}}$ and a class of output objects $\underline{\mathcal{O}}$. Moreover, one
  associates to each of these classes a class of natural maps, the morphisms,  between
  objects. For the problem of HC for example, the input class is the
  set of finite metric spaces and the output class is that of
  dendrograms.  An algorithm is to be regarded as a functor between a category of input objects and a category of output objects. 

 An algorithm will therefore be a procedure that assigns to each $I\in\underline{\mathcal
    I}$ an output $O_I\in\underline{{\mathcal O}}$ with the further property that
  it respects relations between objects in the following sense. Assume
  $I,I'\in \underline{\mathcal I}$ such that there is natural map
  $f:I\rightarrow I'$. Then, the algorithm has to have the property
  that the relation between $O_I$ and $O_{I'}$ has to be represented
  by a natural map for output objects.

\begin{example}
Assume that $\cat{I}$ is such that $\mor_{\cat{I}}(X,Y)=\emptyset$ for all $X,Y\in \cat{I}$ with $X\neq Y$. In this case, since there are no morphisms between input objects any functor $\func{A}:\cat{I}\rightarrow \cat{O}$ can be specified arbitrarily on each $X\in\cat{O}$.
\end{example}

We view any given clustering scheme as a procedure which takes as input a
finite metric space $(X,d_X)$, and delivers as output either an object in $\catc$ or $\catp$:
\begin{itemize}
\item \textbf{Standard clustering}: a pair $(X,P_X)$ where $P_X$ is a partition of $X$. Such a pair is an object in the category $\catc$.  
\item \textbf{Hierarchical clustering}: a pair $(X,\theta_X)$ where $\theta_X$ is a persistent set over $X$. Such a pair is an object in the category $\catp$.
\end{itemize}
  
The concept of \textbf{functoriality} refers to the additional condition
that the clustering procedure should map a pair of input objects into a pair
of output objects in a manner which is consistent with respect
to the morphisms attached to the input and output spaces. When this
happens, we say that the clustering scheme is \textbf{functorial}.
This notion of consistency is made precise in Definition
\ref{def:functor} and described by diagram (\ref{eq:diagram-general}). Let $\many$ stand for any of $\mgen$, $\mmon$ or $\miso$. 

According to Definition \ref{def:functor}, in order to view a standard
clustering scheme as a functor $\cfunc:\many\rightarrow \catc$ we need to specify:

\begin{itemize}
\item[(1)] how it maps
objects of $\many$ (finite metric spaces)
into objects of $\catc$, and 
\item[(2)]
how a morphism $f:(X,d_X)\rightarrow (Y,d_Y)$ between two
objects $(X,d_X)$ and $(Y,d_Y)$ in the input category
$\many$ induces a map in the output category
$\catc$, see diagram (\ref{eq:diagram-general}).
\end{itemize}

\begin{equation}\label{eq:diagram-general}
\xymatrix{
(X,d_X) \ar[d]^{\cfunc} \ar[r]^{f} &  (Y,d_Y) \ar[d]^{\cfunc} \\
 (X,P_X) \ar[r]^{\cfunc(f)} & (Y,P_Y)} 
\end{equation}

Similarly, in order to view a hierarchical clustering scheme as a functor $\hfunc:\many\rightarrow \catp$ we need to specify:
\begin{itemize}
\item[(1)] how it maps
objects of $\many$ (finite metric spaces)
into objects of $\catp$, and 
\item[(2)]
how a morphism $f:(X,d_X)\rightarrow (Y,d_Y)$ between two
objects $(X,d_X)$ and $(Y,d_Y)$ in the input category
$\many$ induces a map in the output category
$\catp$, see diagram (\ref{eq:diagram-general}).
\end{itemize}

\begin{equation}\label{eq:diagram-general}
\xymatrix{
(X,d_X) \ar[d]^{\hfunc} \ar[r]^{f} &  (Y,d_Y) \ar[d]^{\hfunc} \\
 (X,\theta_X) \ar[r]^{\hfunc(f)} & (Y,\theta_Y)} 
\end{equation}

Precise constructions will be discussed in Sections \ref{results-sc} and \ref{results-hc}, where we study different clustering algorithms, both standard or flat, and hierarchical,  using the idea of
functoriality.  

We have 3 possible ``input'' categories ordered by
inclusion (\ref{eq:inclu-cats}). The idea is that studying
functoriality over a larger category will be more stringent/demanding
than requiring functoriality over a smaller one. We will consider
different clustering algorithms and study whether they are functorial
over our choice of the input category. The least demanding one, $\underline{\mathcal
M}^{iso}$ basically enforces that clustering schemes are not dependent on the way points are labeled. We will prove  uniqueness results for functoriality over the most stringent category
$\underline{\mathcal M}^{gen}$, and finally we study how relaxing the
conditions imposed by the morphisms in $\underline{\mathcal M}^{gen}$,
namely, by restricting ourselves to the smaller but intermediate
category $\underline{\mathcal M}^{mon}$, we permit more functorial
clustering algorithms.


\section{Invariants of metric spaces}\label{sec:invariants}
Now we consider maps from a metric space into the ordered extended positive real line $\catr$ that behave well under morphisms of the categories $\miso$, $\mmon$, and $\mgen$, respectively. Such maps will be used later on for defining certain clustering functors, and the richness of such maps, or lack thereof, signals the degree to which these clustering functors exhibit different interesting behaviors. 

\begin{definition}\label{def:invariant}
Let $\many$ be any of $\mgen$, $\mmon$ or $\miso$, then  an \textbf{invariant} is any functor \mbox{$\mathfrak{I}:\many\rightarrow\catr$}.
\end{definition}
Recall that this means that $\mathfrak{I}(X)\geq \mathfrak{I}(Y)$ whenever $\mor_{\many}(X,Y)\neq \emptyset.$ Diagrammatically, this is the behavior we seek:
\begin{equation}\label{eq:diagram-invariants}
\xymatrix{
(X,d_X)\ar[r]^{\phi} \ar[d]_{\ifunc} & (Y,d_Y) \ar[d]^{\ifunc} \\
\ifunc(X) \ar[r]^{\geq} & \ifunc(Y)}.
\end{equation}

\subsection{The case of $\miso$.} The collection of invariants in $\miso$ is the largest possible among $\miso,\mmon$ and $\mgen$, and this could be seen as a consequence of the fact that $\miso$ is the category with the fewest morphisms cf. (\ref{eq:inclu-cats}). In this case all $\miso$-functorial invariants are given in the following manner: $\ifunc(X)=\Psi(d_X)$
where $\Psi:\R^{|X|}_+\times\R^{|X|}_+\rightarrow \mathbb{E}$ is such that $\Psi(A) = \Psi(\pi'\cdot A\cdot \pi)$ for any permutation matrix $\pi$ of size $|X|\times |X|$ and $A\in \R^{|X|}_+\times\R^{|X|}_+$.

\subsection{The case of $\mmon$.} 

\begin{example} A first example of a $\mmon$-invariant is any non-increasing function $\zeta:\N\rightarrow \mathbb{E}$ of the cardinality $|X|$ of a finite metric space $X$: clearly, since $\mor_{\mmon}(X,Y)\neq \emptyset$ requires that $|X|\leq |Y|$, then $\zeta(|X|)\geq \zeta(|Y|)$. Hence, the construction is functorial. 
\end{example}

Another simple example is \mbox{$\ifunc_{\mathrm{sep}}:\mmon\rightarrow \catr$}, which assigns to any finite metric space $X$ 
 its separation $\sep(X).$  
This example belongs to a larger class. 
\begin{example}
For each $k\geq 2$  let $\ifunc_k:\mmon\rightarrow\catr$ be the functor given by 
$$\ifunc_k^-(X):=\inf\big\{\varepsilon\geq 0|\,\mor_{\mmon}(\Delta_k(\varepsilon),X)\neq \emptyset\big\}$$
\mbox{for $X\in\ob(\mmon)$ with $|X|\geq k$}. If $|X|<k$ we put $\ifunc_k(X)=+\infty.$
\end{example}
 Notice that $\ifunc_{sep}=\ifunc_2^-$.

That this definition is indeed $\mmon$-functorial can be seen easily.
\begin{proof} Assume that $X,Y\in\ob(\mmon)$ and $\phi\in\mor_{\mmon}(X,Y)$, then, we need to check that $\ifunc_k^-(X)\geq \ifunc_k^-(Y)$. If $|X|,|Y| < k$, or $|X|\geq k$ and $|Y|<k$, there's nothing to prove. Let's assume that $|X|,|Y|\geq k$. Pick $\varepsilon>\ifunc_k^-(X)$ and $f\in\mor_{\mmon}(\Delta_k(\varepsilon),X)$. Then, $\phi\compcirc f \in \mor_{\mmon}(\Delta_k(\varepsilon),Y)$, and hence $\varepsilon\geq \ifunc_k^-(Y).$ The conclusion follows since $\varepsilon>\ifunc_k^-(X)$ was arbitrary.  
\end{proof}
Another family of $\mmon$-functorial invariants with a similar structure is the following. 
\begin{example}
For each $k\in\N$ define
$$\ifunc_k^+(X):=\sup\big\{\varepsilon\geq 0|\,\mor_{\mmon}(X,\Delta_k(\varepsilon))\neq \emptyset\big\}$$
if $|X|\leq k$, and $\ifunc_k^+(X)=0$ otherwise.
\end{example}
 Finally, these two examples can be generalized in the following manner.
\begin{definition}
 Let  $\Omega\subset \mathcal{M}$ be a collection of finite metric spaces and $\iota:\miso\rightarrow\catr$ any $\miso$-functorial invariant. Then, one defines
$$\mathfrak{I}_\Omega^-(X):=\inf\big\{\iota(\omega)|\,\exists\,\omega\in\Omega\,\mbox{s.t.}\,\mor_{\mmon}(\omega,X)\neq \emptyset\big\}$$
and
$$\mathfrak{I}_\Omega^+(X):=\sup\big\{\iota(\omega)|\,\exists\,\omega\in\Omega\,\mbox{s.t.}\,\mor_{\mmon}(X,\omega)\neq \emptyset\big\}.$$
\end{definition}

In these definitions we use the convention that $\inf$ over an empty set equals $+\infty$ whereas $\sup$ over an empty set equals $0.$ The proof that these these two constructions are $\mmon$-functorial follows the same line of argument as for the case of $\mathfrak{I}_k^-$ above. That these definitions generalize the ones of $\mathfrak{I}^\pm_k$ above can be seen by letting $\Omega=\{\Delta_k(\delta),\,\delta>0\}$ and $\iota$ equal to the separation invariant. 

We do not know a complete classification of all the $\mmon$-functorial invariants but it seems that progress could be made using our techniques. Notice that the family of $\mmon$-functorial invariants is rich in the sense that it is closed under addition and multiplication.

\subsection{The case of $\mgen$.} This case is particularly easy to study: since the collection of all morphisms is quite large (it is a super-set of those of $\miso$ and $\mmon$), one would expect that there are relatively few $\mgen$-functorial invariants. This is indeed the case and one has that the only $\mgen$-functorial invariants  that are possible are the constant ones: as we saw in Remark \ref{remark:mgen}, one property of the category $\mgen$ is that for any $X,Y\in\ob(\mgen)$ there exist morphisms $\phi_1\in\mor_{\mgen}(X,Y)$ and $\phi_2\in\mor_{\mgen}(Y,X)$. Functoriality of $\mathfrak{I}$ would then require that $\mathfrak{I}(X)\geq \mathfrak{I}(Y)\geq \mathfrak{I}(X),$
and hence, since $X,Y$ were arbitrary, $\mathfrak{I}$ must be constant.


\section{Standard Clustering}
\label{results-sc}

\begin{definition}
\label{ex:key} For each $\delta>0$ we define the \textbf{Vietoris-Rips clustering  functor} 
{$$\rfunc_\delta: \mgen \rightarrow
\catc$$}
as follows.  For a finite metric space $(X, d_X)$, we set $\rfunc_\delta(X,d_X)$ to be $(X,P_X(\delta))$, where
$P_X(\delta)$ is the partition of $X$ associated to the equivalence
relation $\sim_\delta$ defined in Example \ref{rips}.   We define how $\rfunc_\delta$ acts on maps $f : (X,d_X) \rightarrow (Y,d_Y)$:
$\rfunc_\delta(f)$ is simply the set map $f$ regarded as a
morphism from $(X,P_X(\delta))$ to $(Y,P_Y(\delta))$ in
$\catc$.  
\end{definition}

We now check that $\rfunc_\delta$ is in fact $\mgen$-functorial. Fix finite metric spaces $(X,d_X)$ and $(Y,d_Y)$ and a morphism $\phi\in\mor_{\mgen}(X,Y)$.  It is enough to prove that whenever $x,x'\in X$ are s.t. $x\sim_\delta x'$, then $\phi(x)\sim_\delta\phi(x')$ in $Y$ as well. Let $x=x_0,x_1,\ldots,x_n=x'$ be points in $X$ with $d_X(x_i,x_{i+1})\leq \delta$ for all $i=0,1,\ldots,n-1$. Then, since $\phi$ is distance non-increasing, $d_Y(\phi(x_i),\phi(x_{i+1}))\leq \delta$ for all $i=0,1,\ldots,n-1$ as well,  and by Definition \ref{def-equiv-rips}, then $\phi(x)\sim_\delta\phi(x')$.\footnote{The Vietoris-Rips functor is actually just \textbf{single linkage clustering} as it is well known, see \cite{carlsson-memoli-full,clust-um}. The name comes from certain simplicial constructions used by Vietoris in the 1930s to study \v Cech cohomology theories of metric spaces. Much later, Rips used similar constructions in the context of geometric group theory, see \cite{hausmann}.}

 By restricting $\rfunc_\delta$ to the subcategories 
$\miso\mbox{ and } \mmon$,
we obtain functors $\rfunc_\delta^{iso}: \miso
\rightarrow \catc$ and $\rfunc_\delta^{inj}:
\mmon \rightarrow\catc$. We will denote all these functors by $\rfunc_\delta$ when there is no ambiguity.

\begin{remark}[\bf The Vietoris-Rips functor is surjective.] Among the desirable conditions singled out by Kleinberg \cite{kleinberg}, one has that of \emph{surjectivity} (which he referred to as ``richness''). Given a finite set $X$ and $P_X\in\mathcal{P}(X)$, surjectivity calls for the existence of a metric $d_X$ on $X$ such that $\rfunc_\delta(X,d_X)=(X,P_X)$. 
Indeed this is the case: pick any $\alpha>1$ and let $d_X(x,x')=\delta$ whenever $x$ and $x'$, $x\neq x'$, are in the same block of $P_X$, and $d_X(x,x')=\alpha \delta$ when $x$ and $x'$ are in different blocks of $P_X$. With this definition it is easy to verify that $d_X$ is an ultrametric on $X$ and hence a metric. That $\rfunc_\delta(X,d_X)=(X,P_X)$ follows directly from Lemma \ref{prop:props-uaster} and the definition of the Vietoris-Rips functor.
\end{remark}

For $\many$ being any one of our choices $\miso$,
$\mmon$ or $\mgen$, a clustering functor in this context will be
denoted by $\mathfrak{C}:\many\rightarrow \catc$. \textbf{Excisiveness} of a clustering functor refers to the property that once a finite metric space has been partitioned by the clustering procedure, it should not be further split by subsequent applications of the same algorithm.

\begin{definition}[\bf Excisive clustering functors]
  We say that a clustering functor $\cfunc$ is \textbf{excisive} if for all
  $(X,d_X) \in \ob(\many)$, if we write $\mathfrak{C}(X,d_X) = (X, \{X_\alpha\}_{\alpha\in
  A})$, then
{ $$\cfunc\left(X_\alpha,{d_X}_{|_{X_\alpha \times X_\alpha}}\right) = (X_\alpha,\{X_\alpha\})\,\,\mbox{for all $\alpha\in A$.}$$}
\end{definition}

\begin{remark}[\bf The Vietoris-Rips functor is excisive]\label{remark:rips-excisive}
Fix $\delta>0$ and consider the functor $\rfunc_\delta$ defined in Example \ref{ex:key}. We claim that $\mathfrak{R}_\delta$ is excisive. Indeed, write $\mathfrak{R}_\delta(X,d) = (X,\{X_\alpha,\,\alpha\in A\})$. Then, $X_\alpha$, $\alpha\in A$ are the different blocks of the partition of $X$ into equivalence classes of $\sim_\delta$ ($x\sim_\delta x'$ if and only if $x,x'\in X_\alpha$ for some $\alpha\in A$). Fix $\alpha\in A$ and let $x,x'\in X_\alpha$. Let $x_0,x_1,\ldots,x_m\in X_\alpha$ be s.t. $x_0=x$, $x_m=x'$ and $d_X(x_i,x_{i+1})\leq \delta$ for $i\in\{0,\ldots,m-1\}$. But then it follows that $x_i\in X_\alpha$ for all $i\in\{0,\dots,m\}$ as well and hence $\mathfrak{R}_\delta\big(X_\alpha,{d_X}_{|_{X_\alpha\times X_\alpha}}\big)=(X_\alpha,\{X_\alpha\}).$
\end{remark}

There exist large families of non-excisive clustering functors in $\mmon$.

\begin{example}[\bf A family of non-excisive functors in $\mmon$]\label{example:non-excisive}

For each invariant $\ifunc:\mmon\rightarrow\catr$ (recall their construction in \S\ref{sec:invariants}) and non-increasing  function $\eta:\mathbb{E}\rightarrow\mathbb{E}$,  consider the clustering functor $\widehat{\rfunc}:\mmon\rightarrow  \catc$ defined as follows: For a finite metric space $(X, d_X)$, we define $\widehat{\rfunc}(X, d_X)$  to be $(X,\widehat{P}_X)$, 
where $\widehat{P}_X$ is the partition of $X$ associated to the equivalence relation $\sim_{\eta(\ifunc(X))}$ on $X$. That $\widehat{\rfunc}$ is a functor follows from the fact that whenever $\phi\in\mor_{\mmon}(X,Y)$ and $x\sim_{\eta(\ifunc(X))} x'$, then $\phi(x)\sim_{\eta(\ifunc(Y))}\phi(x')$. 

Indeed, let $x,x'\in X$ be s.t. $x\sim_{\eta(\ifunc(X))}x'$ and let $x=x_0,\ldots,x_n=x'$ in $X$ s.t. $d_X(x_i,x_{i+1})\leq \eta(\ifunc(X))$ for $i=0,\dots,n-1$. Then, for $i=0,\ldots,n-1$,
$d_Y(\phi(x_i),\phi(x_{i+1}))\leq d_X(x_i,x_{i+1})\leq \eta(\ifunc(X)).$
Since $\eta$ is non-increasing and $\ifunc(X)\geq \ifunc(Y)$, it follows that for $i=0,1,\ldots,n-1$,
$d_Y(\phi(x_i),\phi(x_{i+1}))\leq \eta(\ifunc(Y))$,
and hence $\phi(x)\sim_{\eta(\ifunc(Y))}\phi(x').$


Now, the functors $\widehat{\rfunc}$ are \textbf{not excisive} in general. An explicit example is the following, which applies to $\ifunc=\ifunc_{sep}$ and $\eta(z)=z^{-1}$. Consider the metric space $(X,d_X)$ depicted in figure on the right, where the metric is given by the graph metric on the underlying graph (the edge weights are shown in the figure). Note that $\textrm{sep}(X,d_X)=1/2$ and thus $\eta(\ifunc(X)) = 2.$ We then find that $\widehat{\mathfrak{R}}(X,d_X) = (X,\{\{A,B,C\},\{D,E\}\})$. But, 
$\sep\left(\{A,B,C\},\dmthree{2}{3}{1} \right) = 1$
and hence $\eta(\ifunc({\{A,B,C\}}))=1$. Therefore,
{ $$\widehat{\rfunc}\left(\{A,B,C\},\dmthree{2}{3}{1} \right)=(\{A,B,C\},\{A,\{B,C\}\}),$$}
and we see that $\{A,B,C\}$ gets further partitioned by $\widehat{\rfunc}$.
\begin{figure}
\centering
	\includegraphics[width=1\linewidth]{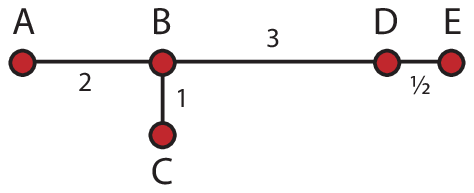}
	\caption{Metric space used to prove that the functor $\widehat{\rfunc}:\mmon\rightarrow \catc$ is not excisive. The metric is given by the graph distance on the graph.}  \label{fig:non-excisive} 
\end{figure}
\end{example}

It is interesting to point out that the same construction of a non-excisive functor in $\mgen$ would not work, since as we saw in \S\ref{sec:invariants}, all the invariants $\mathfrak{I}:\mgen\rightarrow\catr$ are constant. We will see in \S\ref{sec:mgen} that  $\mgen$ permits the existence of only one (well behaved) clustering functor which turns out to be the Vietoris-Rips functor, which is excisive as we saw in Remark \ref{remark:rips-excisive}.

\subsection{The case of \texorpdfstring{$\miso$}{Miso}}\label{sec:miso}
It is easy to describe all $\miso$-functorial clustering schemes.  Let $\mathcal{I}$ denote the collection of all isometry classes of finite metric spaces.  Clearly, one must specify the clustering functor on each class of $\mathcal{I}$ in a manner which is invariant to self-symmetries of (an element of) the class. The precise statement is as follows. For each $\zeta\in\mathcal{I}$ let $(X_\zeta,d_{X_\zeta})$ denote an element of the class $\zeta$, $G_\zeta$ the isometry group of $(X_\zeta,d_{X_\zeta})$,  and $\Xi_\zeta$ the set of all fixed points of the action of $G_\zeta$ on $\mathcal{P}(X_\zeta)$. For completeness we state as a theorem the following obvious fact:
\begin{theorem}[\textbf{Classification of $\miso$-functorial clustering schemes}]\label{theo:miso}
Any $\miso$-functorial clustering scheme determines a choice of $p_\zeta\in\Xi_\zeta$ for each $\zeta\in\mathcal{I}$, and conversely, a choice of $p_\zeta$ for each $\zeta\in\mathcal{I}$ determines an $\miso$-functorial scheme.
\end{theorem}

\subsection{Representable Clustering Functors}
In what follows, $\many$ is either of $\mmon$ or $\mgen$. For each $\delta>0$ the Vietoris-Rips functor $\rfunc_\delta:\many\rightarrow\catc$ can be described in an alternative way which relies on the ability to construct maps from the two point metric space $\Delta_2(\delta)$ into a given input finite metric space $(X,d_X)$. A first trivial observation is that the condition that $x,x'\in X$ satisfy $d_X(x,x')\leq \delta$ is equivalent to requiring the existence of a map $f\in{\mor}_{\many}(\Delta_2(\delta),X)$ with $\{x,x'\}\subset \im(f)$. Using this observation, we can reformulate the condition that $x\sim_\delta x'$ by the requirement that there exist $z_0,z_1,\ldots,z_k\in X$ with $z_0=x$, $z_k=x'$, and $f_1,f_2,\ldots,f_k\in\mor_{\many}(\Delta_2(\delta),X)$ with $\{x_{i-1},x_i\}\subset\im(f_i)$ $\forall i=1,2,\ldots,k$. Informally, this points to the interpretation that $\{\Delta_2(\delta)\}$ is the ``parameter'' in a ``generative model'' for $\rfunc_\delta$.

This immediately suggests considering more general clustering functors constructed in the following manner. Let $\Omega$ be any  fixed collection of finite metric spaces. Define a clustering functor \mbox{$\mathfrak{C}^\Omega:\many\rightarrow\catc$} as follows: let
$(X,d)\in \ob(\many)$ and write $\mathfrak{C}^\Omega(X,d) = (X,\{X_\alpha\}_{\alpha\in
A})$. One declares that points $x$ and $x'$ belong to the same
block $X_\alpha$ if and only if there exist
\begin{itemize}

  \item a sequence of points $z_0,\ldots,z_k\in X$ with $z_0 = x$ and $z_k=x'$,

  \item a sequence of metric spaces $\omega_1,\ldots,\omega_k\in
    \Omega$ and

  \item for each $i=1,\ldots,k$, pairs of points
    $(\alpha_i,\beta_i)\in \omega_i$ and morphisms $f_i\in
    {\mor}_{\many}(w_i,X)$ s.t. $f_i(\alpha_i) = z_{i-1}$ and
    $f_i(\beta_i) = z_i$.

\end{itemize}
Also, we declare that $\cfunc^\Omega(f)=f$ on morphisms $f$. Notice that above one can assume that $z_0,z_1,\ldots,z_k$ all belong to $X_\alpha$.

\begin{remark}\label{remark:rep-no-split}
Let $\Omega$ be a collection of finite metric spaces. Consider the clustering functor \mbox{$\mathfrak{C}^\Omega:\many\rightarrow \catc$} constructed as above. Then, clearly,  $\mathfrak{C}^\Omega(\omega,d_\omega)=(\omega,\{\omega\})$ for all $\omega\in\Omega$. 
As a consequence, we have the following: If $X$ is any finite metric space and $\omega\in\Omega$, then all $f\in \mor_{\many}(\omega,X)$ are s.t. $\im(f)$ is fully contained in a single block of the partition of $X$ by $\mathfrak{C}^\Omega$. 

\end{remark}

\begin{remark}\label{remark:omega}
Let $\Omega\supset \Omega'\neq \emptyset$, then $\mathfrak{C}^{\Omega'}$ produces partitions that are not coarser than those produced by $\mathfrak{C}^{\Omega}$.
\end{remark}

\begin{definition}
  We say that a clustering functor $\mathfrak{C}$ is \textbf{representable} whenever
  there exists a collection of finite metric spaces $\Omega$
  such that $\mathfrak{C} = \mathfrak{C}^\Omega$. In this case, we say that $\mathfrak{C}$ is
  \textbf{represented} by $\Omega$. We say that $\cfunc$ is \textbf{finitely representable} whenever $\cfunc=\cfunc^\Omega$ for some \emph{finite} collection of finite metric spaces $\Omega$.
\end{definition}

As we saw above, the Vietoris-Rips functor $\rfunc_\delta$ is (finitely) represented by $\{\Delta_2(\delta)\}$. 

\subsection{Representability and Excisiveness.} We now prove that representability and excisiveness are equivalent properties. Notice that excisiveness is an axiomatic statement whereas representability asserts existence of generative model for the clustering functor.

\begin{theorem} \label{theo:exc-rep}
  Let $\many$ be either of $\mmon$ or $\mgen$. Then any
  clustering functor on $\many$ is excisive if and only if it is
  representable.
\end{theorem}

\begin{proof}
Assume that $\mathfrak{C}$ is an excisive clustering functor. Let 
\beq{eq:omega}\Omega:=\big\{(B,d_{|_{B\times B}}),\, B\in P,\, \mbox{where $(X,P)=\mathfrak{C}(X,d)$ and $(X,d)\in \mathcal{M}$}\big\}.\eeq
That is, $\Omega$ contains all the possible blocks (endowed with restricted metrics) of all partitions obtained by applying $\mathfrak{C}$ to all finite metric spaces. Notice that by definition of $\Omega$, and since $\mathfrak{C}$ is excisive, 
\beq{eq:excisive-omega}\mathfrak{C}(\omega,d_\omega) = (\omega,\{\omega\})\,\mbox{for all $\omega\in\Omega$}.\eeq 
We will now prove that $\mathfrak{C}$ is represented by $\Omega$, that is $\mathfrak{C}=\mathfrak{C}^\Omega$.  Fix any finite metric space $(X,d_X)$ and write $(X,P)=\mathfrak{C}(X,d_X)$ and $(X,P')=\mathfrak{C}^\Omega(X,d_X)$. We know that $\Omega$ contains all blocks $B\in P$, and since the inclusion $B\hookrightarrow X$ is a morphism in $\mmon$, by Remark \ref{remark:rep-no-split} one sees that  $\mathfrak{C}^\Omega(B,{d_{X}}_{|_{B\times B}})=(B,\{B\})$. Now,  by functoriality of $\mathfrak{C}^\Omega$,  any two points $x,x'$ in $B$ must belong to the same block of $P'$.  Since $B$ was any block of $P$ it follows  that $P$ is a refinement of $P'$.

This is depicted by the diagram below:
\begin{equation}\label{eq:diagram-i}
\xymatrix{
(B,{d_{X}}_{|_{B\times B}}) \ar[r]^{\hookrightarrow} \ar[d]_{\mathfrak{C}^\Omega} & (X,d_X) \ar[d]^{\mathfrak{C}^\Omega} \\
(B,\{B\}) \ar[r]^{\mathfrak{C}^\Omega(\hookrightarrow)} & (X,P')}
\end{equation}


Conversely, assume $x,x'\in B'\in P'$ and let $\omega_1,\ldots,\omega_k\in \Omega$, $(\alpha_i,\beta_i)\in \omega_i$ for $i=1,2,\ldots,k$, $z_0,z_1,\ldots,z_k \in X$  and $f_1,f_2,\ldots,f_k$ s.t. $f_i\in \mor_{\many}(\omega_i,X)$, $f_i(\alpha_i)=z_{i-1}$ and $f_i(\beta_i)=z_i$ for $i=1,2,\ldots,k$. Consider for each $i\in\{1,2,\ldots,k\}$ the diagram on the right, where we have used (\ref{eq:excisive-omega}).  Then, since for each $i\in\{1,\ldots,k\}$ $\alpha_i,\beta_i$ are in the same (unique) block of $\omega_i$, $z_{i-1}$ and $z_i$ have to be in  the same block of $P$.

$$\xymatrix{
(\omega_i,d_{\omega_i}) \ar[r]^{f_i} \ar[d]_{\mathfrak{C}} & (X,d_X) \ar[d]^{\mathfrak{C}} \\
(\omega_i,\{\omega_i\}) \ar[r]^{\mathfrak{C}(f_i)} & (X,P)}.$$


Since this has to happen for all $i$, we conclude that $x=z_0$ and $z_k=x'$ ought to be in the same block of $P$. Thus $P'$ is a refinement of $P$. We have proved that $P$ refines and is refined by $P'$ and therefore these two partitions must be equal. Since $X$ was any finite metric space, we conclude that $\mathfrak{C}$ is represented by $\Omega$ and we are done.

Assume now that $\cfunc$ is representable, i.e. $\mathfrak{C}=\mathfrak{C}^\Omega$ for some collection of finite metric spaces $\Omega$. We will prove that $\cfunc$ is excisive. Fix a finite metric space $(X,d_X)$ and write $\mathfrak{C}^\Omega(X,d_X)=(X,P)$. Pick any block $B$ of $P$ and any two points $x,x'\in B$. Let $\omega_1,\ldots,\omega_k\in \Omega$, 
$f_1\in\mor_{\many}(\omega_1,X),\ldots,f_k\in\mor_{\many}(\omega_k,X);$ $(\alpha_1,\beta_1)\in\omega_1,\ldots,(\alpha_k,\beta_k)\in\omega_k$; and $z_0,z_1,\ldots,z_k$ be s.t. $z_0=x$, $z_k=x'$, $f_i(\alpha_i)=z_{i-1}$, $f_i(\beta_i)=z_i$ for all $i=1,2,\ldots,k$. One can assume that all $z_i\in B$ and moreover, by Remark \ref{remark:rep-no-split},  $\im(f_i)\subseteq B$. Hence, we can regard each $f_i$ as a morphism in $\mor_{\many}(\omega_i,B)$ where $B$ is endowed with the restricted metric. 

By Remark \ref{remark:rep-no-split}, $\cfunc(\omega,d_\omega)=(\omega,\{\omega\})$ for all $\omega\in\Omega$. Hence, by functoriality we have that $x,x'\in B$ are in the same block of $P_B$, where $(B,P_B) = \mathfrak{C}^\Omega(B,{d_X}_{|_{B\times B}})$. Since $x,x'\in B$ were arbitrary we conclude that $P_B$ consists of only one block. The arbitrariness in the choices of  $B\in P$  and $X\in\ob(\many)$ finishes the proof that $\cfunc=\mathfrak{C}^\Omega$ is excisive.
\end{proof}

\subsection{A factorization theorem.} \label{section:factor}
 We now prove that all finitely representable clustering functors in $\mgen$ and $\mmon$ arise as the composition of the Vietoris-Rips functor with a functor that changes the metric. 

For a given collection $\Omega$ of finite metric spaces let 
\beq{eq:Tfunc}\tfunc^\Omega:\many\rightarrow \many
\eeq be the endofunctor that assigns to each finite metric space $(X,d_X)$ the metric space $(X,d_X^\Omega)$ with the same underlying set and metric given by $d_X^\Omega=\mathcal{U}(W_X^\Omega)$,\footnote{Recall (\ref{eq:U}), the definition of $\mathcal{U}$.} where $W_X^\Omega:X\times X\rightarrow \R_+$ is given by
\beq{eq:W}(x,x')\mapsto \inf\big\{\lambda>0|\,\exists w\in\Omega\,\mbox{and}\,\phi\in\mor_{\many}(\lambda\cdot\omega,X)\,\mbox{with $\{x,x'\}\subset\im(\phi)$}\big\},\eeq
for $x\neq x'$, and by $0$ on $\mathrm{diag}(X\times X)$. Above we assume that the $\inf$ over the empty set equals $+\infty$. Note that $W_X^\Omega(x,x')<\infty$ for all $x,x'\in X$ as long as $|\omega|\leq |X|$ for some $\omega\in \Omega$. Also, $W_X^\Omega(x,x')=\infty$ for all $x\neq x'$ when $|X|<\inf\{|\omega|,\,\omega\in\Omega\}.$

\begin{theorem}\label{theo:comp-rips}
Let $\many$ be either $\mgen$ or $\mmon$ and $\cfunc$ be any $\many$-functorial finitely representable clustering functor represented by some $\Omega\subset \mathcal{M}$. Then, $\cfunc=\rfunc_1\compcirc\tfunc^\Omega$.
\end{theorem}
\begin{proof}
Fix any finite metric space $(X,d_X)$ and write $\cfunc(X,d_X)=(X,P)$ and $\rfunc_1\compcirc\tfunc^\Omega(X,d_X)=(X,P').$
Assume that $x,x'\in B\in P$ and let $x=z_0,z_1,\ldots,z_k=x'\in X$, $\omega_1,\ldots,\omega_k\in \Omega$, $f_i\in\mor_{\many}(\omega_i,X)$ and $(\alpha_i,\beta_i)\in\omega_i$ s.t. $f_i(\alpha_i)=z_{i-1}$, $f_i(\beta_i)=z_i$ for $i=1,2,\ldots,k.$ Then, from (\ref{eq:W}) for all $i=1,2,\ldots,k$, $W_X^\Omega(z_{i-1},z_{i})\leq 1$ and in particular $\max_i W_X^\Omega(z_{i-1},z_i)\leq 1$. It follows from (\ref{eq:U}) that $d_X^\Omega(x,x')\leq 1$ and consequently $x$ and $x'$ are in the same block of $P'$. Thus we see that $P$ refines $P'$.

Assume now that $x,x'\in B'\in P'$. By definition of the Vietoris-Rips functor, $x\sim_1 x'$ in $(X,d_X^\Omega)$,  which implies the existence of $x=x_0,x_1,\ldots,x_k=x'$ in $X$ with $d_X^\Omega(x_i,x_{i+1})\leq 1$ for all $i=1,\ldots,k-1$. Now, by (\ref{eq:U}), for each $i=0,1,\ldots,k$ there exist $y_0^{\scriptscriptstyle(i)},y_1^{\scriptscriptstyle (i)},\ldots,y_{k_i}^{\scriptscriptstyle(i)}$ in $X$ with $y_0^{\scriptscriptstyle(i)}=x_i$, $y_{k_i}^{\scriptscriptstyle(i)}=x_{i+1}$ and $W_X^\Omega\big(y_{j}^{\scriptscriptstyle(i)},y_{j+1}^{\scriptscriptstyle(i)}\big)\leq 1$ for $j=0,1,\ldots,k_i-1$. By concatenating all sequences $\big\{y_j^{\scriptscriptstyle(i)}\big\}_{j=1}^{k_i}$ for $i=0,1,\ldots,k$, we obtain a sequence $z_0,z_1,\ldots, z_N \in X$ with $z_0=x$, $z_N=x'$ and $W_X^\Omega(z_i,z_{i+1})\leq 1$.  By definition of $W_X^\Omega$ and because $\Omega$ is finite,  one sees that for every $i=1,2,\ldots, N$ there exists $\lambda_i\in(0,1]$, $\omega_i\in\Omega$ and $\phi_i\in\mor_{\many}(\lambda_i\cdot\omega_i,X)$ with $\{z_{i-1},z_i\}\subset\im(\phi_i).$ But then, obviously, $\phi_i\in\mor_{\many}(\omega_i,X)$ as well, and since $\cfunc$ is represented by $\Omega$ one sees that $x$ and $x'$ must be in the same block of $P.$ This proves that $P'$ refines $P$, and finishes the proof.
\end{proof}

\subsection{The case of \texorpdfstring{$\mgen$}{Mgen}}\label{sec:mgen}

\subsubsection{A Uniqueness theorem.} In $\mgen$ clustering functors are very restricted, as reflected by the following theorem.
\begin{theorem}\label{theo:unique}
Assume that $\mathfrak{C}:\mgen\rightarrow\catc$ is a clustering functor for which there exists $\delta_{\mathfrak{C}}> 0$ with the property that 
\begin{itemize}
\item $\mathfrak{C}(\Delta_2(\delta))$ is in one piece for all $\delta\in [0,\delta_\mathfrak{C}]$, and
\item $\mathfrak{C}(\Delta_2(\delta))$ is in two pieces for all $\delta> \delta_\mathfrak{C}$.
\end{itemize}
Then, $\cfunc$ is the Vietoris-Rips functor with parameter $\delta_{\mathfrak{C}}$. i.e. $\cfunc=\rfunc_{\delta_\cfunc}$.
\end{theorem}

Recall that the Vietoris-Rips functor is excisive, cf. Remark \ref{remark:rips-excisive}. 
\begin{remark}[About uniqueness in $\mgen$]\label{app:rips-strict}
For each $\delta>0$ consider the clustering functor $\stackrel{\circ}{\rfunc}_\delta$ that assigns points $x,x'$ in a metric space $(X,d_X)$ to the same block if and only if there exist $x_0,x_1,\ldots,x_k\in X$ with $x_0=x$, $x_k=x'$ and $d_X(x_i,x_{i+1})<\delta$ for $i=0,1,\ldots,k-1$. Note the \emph{strict} inequality, cf. with the definition of the Vietoris-Rips functor in Example \ref{ex:key}.

If one changes the conditions in Theorem \ref{theo:unique} in a way such that 
\begin{itemize}
\item $\mathfrak{C}(\Delta_2(\delta))$ is in one piece for all $\delta\in [0,\delta_\mathfrak{C})$, and
\item $\mathfrak{C}(\Delta_2(\delta))$ is in two pieces for all $\delta\geq \delta_\mathfrak{C}$,
\end{itemize}

then, $\cfunc=\stackrel{\circ}{\rfunc}_{\delta_\cfunc}$.
\end{remark}

\begin{proof}[Proof of Theorem \ref{theo:unique}]
Fix a finite metric space $(X,d_X)$ and write $(X,P)=\cfunc(X,d_X)$ and $(X,P')=\rfunc_{\delta_\cfunc}(X,d_X).$

\noindent
(1) Let $x,x'\in B$ for some block $B\in P'$. We will prove that then $x,x'$ are contained in the same block of $P$. Let $x_0,x_1,\ldots,x_k\in X$ be s.t. $x_0=x$, $x_k=x'$ and $d_X(x_i,x_{i+1})\leq \delta_\cfunc$ for all $i\in\{1,\ldots,k-1\}$. We now prove that each pair $\{x_i,x_{i+1}\}$ belongs to the same block of $P'$, what will yield the claim. Indeed, since for each $i$, $d_X(x_i,x_{i+1})\leq \delta_\cfunc$, then for each $i$ there exists  $f_i\in\mor_{\mgen}(\Delta_2(\delta_\cfunc),X)$ with $\im(f_i) = \{x_i,x_{i+1}\}$. Thus, by functoriality and the definition of $\delta_\cfunc$, $x_i$ and $x_{i+1}$ must fall in the same block of $P'$. This concludes the first half of the proof.

\noindent
(2) Assume that $x_0,x_0'\in X$ are in different blocks of $P'$. We will prove that then $x_0$ and $x_0'$ must belong to different blocks of $P$.
Let $\delta := \min\big\{d_X(x,x'),\,x\in X\backslash[x_0']_{\delta_\cfunc},\,x'\in[x_0']_{\delta_\cfunc}\big\}.$
Notice that by Lemma \ref{prop:props-uaster}, $\delta>\delta_\cfunc$. Write 
$\Delta_2(\delta) = \Big(\{p,q\},\dmtwo{\delta}\Big)$
and define the map \mbox{$\phi:X\rightarrow \Delta_2(\delta)$} given by  $\phi(x) = p$ for all $x\in X\backslash [x_0']_{\delta_\cfunc}$, and $\phi(x') = q$ for all $x'\in [x_0']_{\delta_\cfunc}$. We claim that $\phi\in\mor_{\mgen}(X,\Delta_2(\delta))$. Indeed, if $x,x'\in X\backslash [x_0']_{\delta_\cfunc}$ or $x,x'\in [x_0']_{\delta_\cfunc}$, $\phi(x)=\phi(x')$ and there's nothing to check. Assume then that $x\in X\backslash [x_0']_{\delta_\cfunc}$ and $x'\in [x_0']_{\delta_\cfunc}$, and note that by definition of  $\delta$, 
$d_X(x,x')\geq\delta$. On the other hand, $p=\phi(x)$ and $q=\phi(x')$ are at distance $\delta$ from each other. Thus $\phi$ is distance non-increasing. Now, we apply $\cfunc$ to $X$ and $\Delta_2(\delta)$ and note that in this case, by our assumption on $\cfunc$, since $\delta>\delta_\cfunc$, then  $\cfunc(\Delta_2(\delta))$ is in two pieces. By functoriality we conclude that $x_0$ and $x_0'$ cannot lie in the same block of $P.$
\end{proof}

\subsection{Scale invariance in $\mgen$ and $\mmon$.} It is interesting to consider the effect of imposing Kleinberg's scale invariance axiom on $\mgen$-functorial and  $\mmon$-functorial clustering schemes. It turns out that in $\mgen$ there are only two possible clustering schemes enjoying scale invariance, which turn out to be the trivial ones: 
\begin{theorem}\label{theorem:scale-invariance}
Let $\cfunc:\mgen\rightarrow\catc$ be a clustering functor s.t. $\cfunc\compcirc \sigma_\lambda=\cfunc$ for all $\lambda>0$. Then, either
\begin{itemize}
\item $\cfunc$ assigns to each finite metric space $X$ the partition of $X$ into singletons, or
\item $\cfunc$ assigns to each finite metric the partition with only one block.
\end{itemize}
\end{theorem}

\begin{proof}
Let $\cfunc:\mgen\rightarrow\catc$ be a scale invariant clustering functor. Consider first the metric space $\Delta_2(1)$ consisting of two points at distance $1$. There are two possiblities: either \textbf{(1)}  $\cfunc(\Delta_2(1))$ is in one piece, or \textbf{(2)}  $\cfunc(\Delta_2(1))$ in two pieces.  

Assume \textbf{(1)}.  By scale invariance, the partition of $\Delta_2(\delta)$ given by  $\cfunc(\Delta_2(\delta))$ is in one piece for all $\delta>0$. Fix any finite metric space $(X,d_X)$ and write $\cfunc(X,d_X) = (X,P)$. Pick any $x,x'\in X$, let $\delta = d_X(x,x')$, and  consider the morphism $\phi\in\mor_{\mgen}(\Delta_2(\delta),X)$ s.t. $\phi(p)=x$ and $\phi(q)=x'$, where $p$ and $q$ are the two elements of $\Delta_2(\delta).$ By scale invariance and  functoriality, $x$ and $x'$ must then be in the same block of $P$. Since $x$ and $x'$ were arbitrary, $P$ must be in one piece as well. It follows from the arbitrariness of $X$ that $\cfunc$ then assigns the single block partition to any finite metric space.

Now assume \textbf{(2)}. By scale invariance $\cfunc(\Delta_2(\delta))$ is in two pieces for all $\delta>0$.   Pick any finite metric space $(X,d_X)$ and write $\cfunc(X,d_X) = (X,P)$. Pick any $x,x'\in X$, and a partition of $X$ into subsets $A$
 and $B$ such that $x\in A$ and $x'\in B$. Let $\delta =\sep(X)$, and  consider the map $\psi:X\rightarrow \Delta_2(\delta)$ given by $\psi(a)=p$ for all $a\in A$ and $\psi(b)=q$ for all $b\in B$. It is easy to verify that  $\psi\in\mor_{\mgen}(X,\Delta_2(\delta))$. By scale invariance and  functoriality, $x$ and $x'$ cannot be in the same block of $P$. Since $x$ and $x'$ were arbitrary, no two points belong to the same block of $P$, and hence $\cfunc$ assigns $X$ the partition into singletons. It follows from the arbitrariness of $X$ that $\cfunc$ assigns to every finite metric space its partition into singletons. This finishes the proof.
\end{proof}

By refining the proof of the previous theorem, we find that the behavior of any $\mmon$-functorial clustering functor is also severely restricted in each $\mathcal{M}_k$, $k\in\N$.\footnote{Recall that $\mathcal{M}_k$ denotes the collection of all $k$-point metric spaces.}  
\begin{theorem}\label{theorem:scale-invariance-mmon}
Let $\cfunc:\mmon\rightarrow\catc$ be a clustering functor s.t. $\cfunc\compcirc \sigma_\lambda=\cfunc$ for all $\lambda>0$. Let $$K(\cfunc):=\big\{k\geq 2;\,\cfunc(\Delta_k(1))\,\mbox{is in one piece}\big\}.$$

\begin{itemize}
\item If $K(\cfunc)=\emptyset$, then $\cfunc$ assigns to each finite metric space $X$ its partition into singletons.
\item Otherwise, let $k_\cfunc=\min K(\cfunc)$. Then,
\begin{itemize}

\item For each $k\geq k_\cfunc$, $\cfunc$ assigns to each finite metric space $X\in \mathcal{M}_k$ the partition of $X$ with only one block.

\item For each $2\leq k< k_\cfunc$, $\cfunc$ assigns to each finite metric space $X\in \mathcal{M}_k$ the partition of $X$ into $k$ singletons.

\end{itemize}
\end{itemize}
\end{theorem}
\begin{proof}

Assume first that $K(\cfunc)=\emptyset$, pick any finite metric space $X$ and  let $k=|X|$. Pick any two distinct points $x,x'\in X$. Choose any injective map $\phi:X\rightarrow \Delta_k(\sep(X))$ such that $\phi(x)$ and $\phi(x')$ fall in different blocks of the partition of $\Delta_k(\sep(X))$ produced by $\cfunc$. Then, functoriality guarantees that $x$ and $x'$ must belong to different blocks of the partition of $X$ generated by $\cfunc$. Since $x,x'\in X$ were arbitrary we conclude that $\cfunc$ partitions $X$ into singletons.

Now assume that $K(\cfunc)\neq\emptyset$. Pick $k\geq k_\cfunc$ and $X\in\mathcal{M}_k$.  Consider any morphism $\phi\in\mor_{\mmon}(\Delta_k(\diam(X)),X)$ arising from arbitrarily choosing an injection from the underlying set of $\Delta_k(1)$ to $X$. Then, by functoriality, this forces $\cfunc(X,d_X)$ to produce a partition of $X$ into exactly one block.

Finally, for $2\leq k<k_\cfunc$, $\cfunc$ partitions $\Delta_k(1)$ into more than piece. An argument similar to the one given for the case $K(\cfunc)=\emptyset$ proves that then $\cfunc$ partitions any $X\in\mathcal{M}_k$ into ($k$) singletons. 

\end{proof}

\subsection{Clustering functors on \texorpdfstring{$\mmon$ and density}{Mmon}}\label{section:mmon-density}
In the case of $\mmon$, there is a richer class of
allowed clustering schemes that includes functors other than the Vietoris-Rips one. We already saw in Example \ref{example:non-excisive} that in $\mmon$ there are non-excisive clustering schemes that by construction are different from the Vietoris-Rips functor. A class of excisive $\mmon$-functorial clustering schemes which has practical interest is the following.
\begin{figure}[htb] \centering
	\includegraphics[width=1\linewidth]{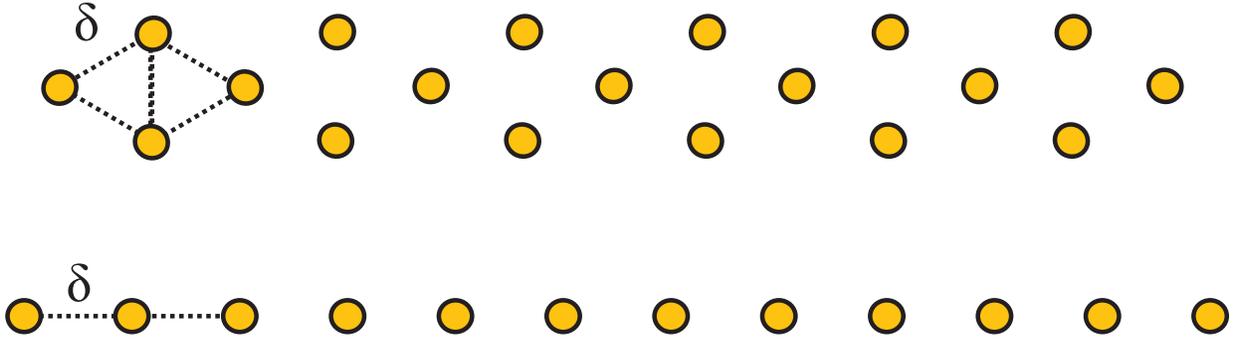}
	\caption{Fix $\delta>0$ and consider the two clustering functors $\mathfrak{C}^{\Delta_2(\delta)}$ and $\mathfrak{C}^{\Delta_3(\delta)}$ on $\mmon$. Then $\mathfrak{C}^{\Delta_2(\delta)}$ applied to both metric spaces in the Figure will yield a single cluster. Meanwhile, $\mathfrak{C}^{\Delta_3(\delta)}$ will yield a single cluster for the metric space on right, but it will not connect the points of the metric space on the left. This admits the interpretation that in order for $\cfunc^{\Delta_3(\delta)}$ to produce a cluster, a certain degree of agglomeration of the points is necessary. This reflects the fact that this clustering scheme is more sensitive to density than its two-point counterpart $\cfunc^{\Delta_2(\delta)}=\rfunc_\delta$. Similar considerations apply to the general case of $\cfunc^{\Delta_k(\delta)}$, any $k\in\N$}  \label{fig:ex-rep} \end{figure}
Consider for each $n\in\mathbb{N}$ and
  $\delta\in \mathbb{R}_+$ the clustering functor
  $\mathfrak{C}^{\Delta_{n}(\delta)}$, and  notice that this clustering method implements the notion that
  in order to deem a two points $x$ and $x'$ in a finite metric space $X$ to be in the same cluster one requires
  that they are ``tightly packed together'' in the sense that there exists a sequence of points $x_0,x_1,\ldots,x_k$ connecting $x$ to $x'$ such that each consecutive pair $\{x_i,x_{i+1}\}$ is contained in a subset $S_i$ of $X$ of cardinality $n$ whose inter-point distances do not exceed $\delta$. This notion is closely related to proposals in network clustering \cite{hungarian-clique} and the DBSCAN algorithm \cite{dbscan}, and for $n>2$ provides a way to tackle the so called ``chaining effect'' exhibited by single linkage clustering (i.e. $\cfunc^{\Delta_2(\delta)}$), see Figure \ref{fig:ex-rep}. Note that Theorem \ref{theo:comp-rips} implies that for a given finite metric space $(X,d_X)$, $\cfunc^{\Delta_3(\delta)}(X,d_X)$ can be obtained by first applying a change of metric to $X$ and then applying the standard Vietoris-Rips functor to the resulting metric space. The change of metric specified by the theorem leads to finding all $3$-cliques in $X$ whose diameter is not greater than $\delta$. 

\begin{corollary}\label{coro:change-metric}
On $\mmon$, for all $m\in \N$
$$\cfunc^{\Delta_m(\delta)} = \rfunc_\delta\compcirc \tfunc^m.$$
\end{corollary}
Above, we have used the shorthand notation $\tfunc^m = \tfunc^\Omega$, for $\Omega=\{\Delta_m(\delta),\,\delta\geq 0\}$.

The formalism of representable clustering functors allows for far more generality than that provided by families $\{\Delta_n(\delta),\,n\in\N,\,\delta>0\}$. For example, one could encode the requirement that clusters be densely connected by considering a clustering functor represented by the metric space in Figure \ref{fig:subdiv}t, which consists of  the vertices of an equilateral triangle of side $\delta$, together with its center.
\begin{figure}
\centering
	\includegraphics[width=.5\linewidth]{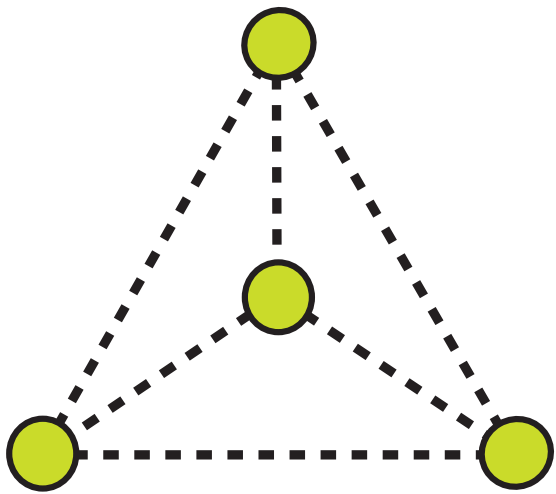}
\caption{A metric space that encodes a notion of ``density'' different from that provided by $\Delta_3(\delta)$. \label{fig:subdiv}. }
\end{figure}

\section{Hierarchical Clustering}
\label{results-hc}

\begin{example}[\textbf{A hierarchical version of the Vietoris-Rips functor}] \label{ex:key} We define a functor 
$$\rfunc: \underline{{\mathcal M}}^{gen} \rightarrow
\underline{{\mathcal P}}$$
as follows.  For a finite metric space $(X, d_X)$, we define $(X,d_X)$ to be the persistent set $(X, \theta_X ^{{ \mathrm{VR}}})$, where
$\theta_X^{{ \mathrm{VR}}}(r)$ is the partition associated to the equivalence
relation $\sim _r$ defined in Example \ref{rips}.  This is clearly an
object in $\underline{{\mathcal P}}$.  We also define how $\rfunc$ acts on maps $f : (X,d_X) \rightarrow (Y,d_Y)$: The value of
$\rfunc(f)$ is simply the set map $f$ regarded as a
morphism from $(X,\theta_X^{{ \mathrm{VR}}})$ to $(Y,\theta_Y^{{ \mathrm{VR}}})$ in
$\underline{{\mathcal P}}$.  That it is a morphism in $\underline{{\mathcal
P}}$ is easy to check. 

Clearly, this functor implements the hierarchical version of single linkage clustering in the sense that for each $\delta\geq 0$, if one writes $\rfunc_\delta{(X,d_X)}=(X,P_X(\delta))$, then $P_X(\delta) = \theta_X^{{ \mathrm{VR}}}(\delta)$.
\end{example}

\subsection{Functoriality over $\miso$} 
This is the smallest category we will deal with.  The morphisms in
$\underline{{\mathcal M}}^{iso}$ are simply the bijective maps between
datasets which preserve the distance function.  As such, functoriality
of a clustering algorithms over $\underline{{\mathcal M}}^{iso}$ simply means
that the output of the scheme doesn't depend on any artifacts in the
dataset, such as the way the points are named or the way in which they
are ordered.  

    {{ Agglomerative hierarchical clustering}, in standard
        form, as described for example in \cite{clusteringref}, begins
        with point cloud data and constructs a \emph{binary} tree (or
        dendrogram) which describes the merging of clusters as a
        threshold is increased.  The lack of functoriality comes from
        the fact that when a single threshold value corresponds to
        more than one data point, one is forced to choose an ordering
        in order to decide which points to ``agglomerate'' first.
        This can easily be modified by relaxing the requirement that
        the tree be binary. This is what we did in Example
        \ref{hierarchical} In this case, one can view these methods as
        functorial on $\underline{{\mathcal M}}^{iso}$, where the functor
        takes its values in arbitrary rooted trees.  It is understood
        that in this case, the notion of morphism for the output
        ($\underline{\mathcal P}$) is simply isomorphism of rooted
        trees. In contrast, we see next that amongst these methods,
        when we impose that they be functorial over the larger (more
        demanding) category $\underline{\mathcal M}^{gen}$ then only
        one of them passes the test.\footnote{The result in Theorem
          \ref{thm:uniq} is actually more powerful in that it states
          that there is a unique functor from $\underline{\mathcal
            M}^{gen}$ to $\underline{\mathcal P}$ that satisfies
          certain natural conditions.} }

\subsection{Complete and Average linkage are not functorial over $\mmon$}

It is interesting to point out
why complete linkage and average linkage (agglomerative) clustering,
as defined in Example \ref{hierarchical} fail to  be functorial on
$\underline{\mathcal M}^{inj}$ (and therefore also on $\mgen$ as well) . A simple example explains this: consider
the metric spaces $X=\{A,B,C\}$ with metric given by the edge lengths
$\{4,3,5\}$ and $Y=(A',B',C')$ with metric given by the edge lengths
$\{4,3,2\}$, as given in Figure
\ref{cexample}. Obviously the map $f$ from $X$ to $Y$ with $f(A)=A'$, 
$f(B)=B'$ and $f(C)=C'$ is a morphism in $\underline{\mathcal
M}^{inj}$. Note that for example for $r=3.5$ (shaded regions of the
dendrograms in Figure \ref{cexample}) we have that the partition of
$X$ is $\Pi_X = \{\{A,C\},B\}$ whereas the partition of $Y$ is $\Pi_Y
= \{\{A',B'\},C'\}$ and thus $f^*(\Pi_Y)=\{\{A,B\},\{C\}\}$. Therefore
$\Pi_X$ does not refine $f^*(\Pi_Y)$ as required by functoriality. The same
construction yields a counter-example for average linkage.
\begin{figure}[htb]
		\centering
		\includegraphics[width=.7\linewidth]{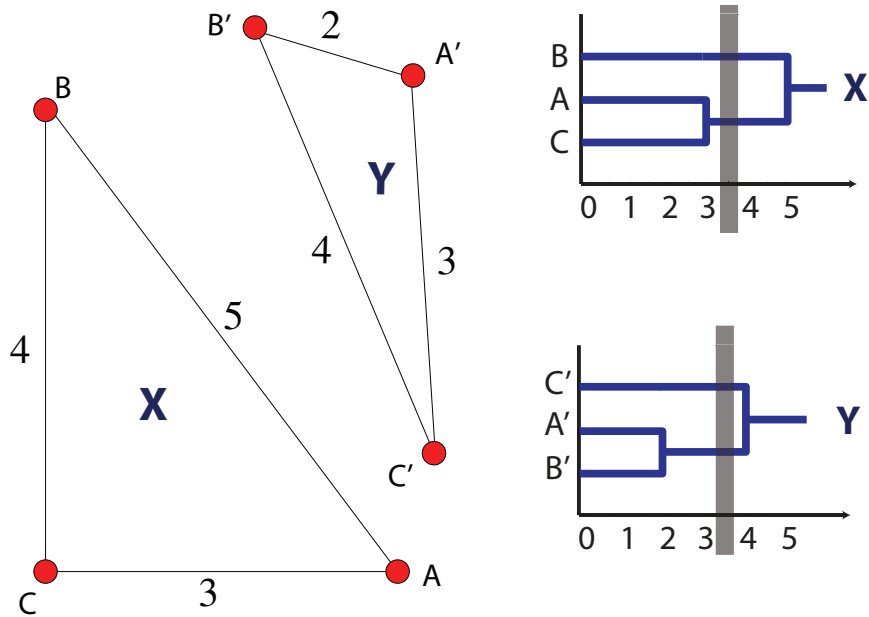}
		\caption{An example that shows why complete linkage
		fails to be functorial on $\underline{\mathcal M}^{inj}$ (and therefore also on $\mgen$).}
		\label{cexample}
\end{figure}

\subsection{Functoriality over $\mgen$: a uniqueness theorem}

We prove a theorem of the same flavor as the
main theorem of \cite{kleinberg}, except that we prove existence and
uniqueness on $\underline{\mathcal M}^{gen}$ instead of impossibility
in our context.

\begin{theorem} \label{thm:uniq} Let $\hfunc: \underline{{\mathcal M}}^{gen} \rightarrow \underline{{\mathcal P}} $
be a hierarchical clustering functor which satisfies the following conditions.
\begin{description}

\item[(I)] {Let $\alpha : \underline{{\mathcal M}}^{gen} \rightarrow
    \underline{Sets}$ and $\beta :\underline{ {\mathcal P}}\rightarrow
    \underline{Sets}$ be the forgetful functors $(X,d_X) \rightarrow
    X$ and $(X,\theta_X) \rightarrow X$, which forget the metric and
    persistent set respectively, and only ``remember'' the underlying sets
    $X$.  Then we assume that $\beta \compcirc \hfunc = \alpha$.  This
    means that the underlying set of the persistent set associated to
    a metric space is just the underlying set of the metric space. }

	\item[(II)] {For $\delta\geq 0$ let
	$\Delta_2(\delta)=(\{p,q\},\dmtwo{\delta})$ denote the two point
	metric space with underlying set $\{ p,q \}$, and where
	$\mbox{dist}(p,q) = \delta$.  Then $\hfunc (\Delta_2(\delta))$ is the
	persistent set $(\{p,q\},\theta_{\Delta_2(\delta)})$ whose
	underlying set is $\{p,q \}$ and where $\theta_{\Delta_2(\delta)}
	(t) $ is the partition with one element blocks when $t <
	\delta$ and is the partition with a single two point block
	when $t \geq \delta$.  }

	\item[(III)] { Write
	$\hfunc(X,d_X)=(X,\theta^{\hfunc})$, then for any $t <
	\mbox{sep}(X)$, the partition $\theta^{\hfunc}(t)$ is the
	discrete partition with one element blocks.  }
\end{description}
Then $\hfunc $ is equal to the functor $\rfunc$.  
\end{theorem}

We should point out that another characterization of single linkage
has been obtained in the book \cite{jardine-sibson}.

\begin{proof}[Proof of Theorem \ref{thm:uniq}]
Write $\hfunc(X,d_X) = (X,\theta_X^{\hfunc})$. For each $r\geq 0$ we will
prove that \textbf{(a)} $\theta_X^{{ \mathrm{VR}}}(r)$ is a refinement of $\theta_X^{\hfunc}(r)$
and \textbf{(b)} $\theta_X^{\hfunc}(r)$ is a refinement of $\theta_X^{{ \mathrm{VR}}}(r)$. Then it will follow that $\theta_X^{{ \mathrm{VR}}}(r)=\theta_X^{\hfunc}(r)$ for all
$r\geq 0$, which shows that the objects are the same.  Since this is a situation where, given any pair of objects, there is at most one  morphism    between them,  this also determines the effect of the functor on morphisms.  


Fix $r\geq 0$. In order to obtain (a) we need to prove that whenever
$x,x'\in X$ lie in the same block of the partition $\theta_X^{{ \mathrm{VR}}}(r)$,
that is $x\sim_r x'$, then they both lie in the same block of
$\theta_X^{\hfunc}(r)$.

skip
It is enough to prove the following
\begin{claim}  Whenever
 $d_X(x,x')\leq r$ then $x$ and $x'$ lie in the same block of $\theta_X^{\hfunc}(r)$.
\end{claim}

Indeed, if the claim is true, and $x\sim_r x'$ then one can find
$x_0,x_1,\ldots,x_n$ with $x_0=x$, $x_n=x'$ and $d_X(x_i,x_{i+1})\leq
r$ for $i=0,1,2,\ldots,n-1$. Then, invoking the claim for all pairs
$(x_i,x_{i+1})$, $i=0,\ldots,n-1$ one would find that: ${x}=x_0$ and
$x_1$ lie in the same block of $\theta_X^{\hfunc}(r)$, $x_1$ and
$x_2$ lie in the same block of $\theta_X^{\hfunc}(r)$, $\ldots$, $x_{n-1}$
and $x_n=x'$ lie in the same block of $\theta_X^{\hfunc}(r)$. Hence, $x$
and $x'$ lie in the same block of $\theta_X^{\hfunc}(r)$.

So, let's prove the claim. Assume $d_X(x,x')\leq r$, then the function
given by $p \rightarrow x, q \rightarrow x^{\prime}$ is a morphism
$g:\Delta_2(r) \rightarrow ( X, d_X) $ in $\underline{{\mathcal M}}^{gen}$.
This means that we obtain a morphism $$\hfunc (g) : \hfunc (\Delta_2(r))
\rightarrow \hfunc (X,d_X)$$ in $\underline{{\mathcal P}}$.  But, by assumption (\textrm{II}),  $p$ and
$q$ lie in the same block of the partition $\theta_{\Delta_2(r)}(r)$. By
functoriality it follows that
$\hfunc(g)$ is persistence preserving 
and hence the elements $g(p) = x$ and $g(q) = x^{{\prime}}$ lie in the
same block of $ \theta_X^{\mathfrak{H}} (r)$. This concludes the proof of (a).

skip
For condition (b), assume that $x$ and $x'$ belong to the same block
of the partition $\theta_X^{\hfunc}(r)$ for some $r\geq 0$. We will prove that necessarily
$x\sim_r x'$. This of course will imply that $x$ and $x'$ belong to the
same block of $\theta_X^{{ \mathrm{VR}}}(r)$.

skip
Consider the metric space $(X_r,d_{r})$ whose points are the
equivalence classes of $X$ under the equivalence relation $\sim _r$,
and where the distance $d_{r}({\mathcal B},{\mathcal B'})$ between
two equivalence classes ${\mathcal B}$ and ${\mathcal B}' $ is defined to be
the maximal ultra-metric pointwisely less than or equal to 
$$W_X(\mathcal{B},\mathcal{B}'):= \min_{x
\in {\mathcal B}}\min _{x^{\prime} \in {\mathcal B}'} d_X
(x,x^{\prime}),$$
that is $d_r=\mathcal{U}(W_X)$, see (\ref{eq:U}). It follows from the definition of $\sim _r$ and Lemma \ref{prop:props-uaster}
that if two equivalence classes are distinct, then the distance
between them is $> r$. Indeed, for otherwise there would be $x,x'\in X$ with $[x]_r\neq [x']_r$ and $d_r([x]_r,[x']_r)\leq r$, which would imply that $W_X([x]_r,[x']_r)\leq r$, which is a contradiction. Hence $\mbox{sep}(X_r,d_r)>r.$

Write $\hfunc (X_r,d_{r})=(X_r,\theta_{X_r}^\hfunc)$. Since
$\mbox{sep}(X_r)>r$, hypothesis (III) now directly shows that the
blocks of the partition $\theta_{X_r}^\hfunc(r)$ are exactly the equivalence
classes of $X$ under the equivalence relation $\sim _r$, that is
$\theta_{X_r}^\hfunc(r) = \theta_X^{{ \mathrm{VR}}}(r)$. Finally, consider the morphism
$$\pi_r: (X,d_X)
\rightarrow (X_r,d_{r})$$ in $\underline{{\mathcal M}}^{gen}$ given on elements $x
\in X$ by $\pi_r (x) =[x]_r$, where $[x]_r$ denotes the equivalence class of
$x$ under $\sim_r$. By functoriality, $\hfunc(\pi_r):(X,\theta_X^{\hfunc})\rightarrow (X_r,\theta_{X_r}^\hfunc)$ is
persistence preserving, and therefore, $\theta^{\hfunc}(r)$ is a
refinement of $\theta_{X_r}(r)=\theta_X^{{ \mathrm{VR}}}(r)$. This is depicted as follows: 
\begin{displaymath}
\xymatrix{
(X,d_X) \ar[r]^\hfunc \ar[d]_{\pi_r} & (X,\theta_X^{{ \mathrm{VR}}}) \ar[d]^{\hfunc(\pi_r)} \\
(X_r,d_{r}) \ar[r]^{\hfunc} & (X_r,\theta_{X_r}^\mathfrak{H}) }
\end{displaymath}

This concludes the proof of (b).
\end{proof}

\subsubsection{Comments on Kleinberg's conditions}\label{sec:klein-conds}

\noindent We conclude this section by observing that analogues of  the three (axiomatic) 
properties considered by Kleinberg in \cite{kleinberg} hold for $\rfunc$.

\begin{itemize}
	 \item Kleinberg's {first condition} was {\em scale-invariance}, which
	 asserted that if the distances in the underlying point cloud
	 data were multiplied by a constant positive multiple
	 $\lambda$, then the resulting clustering decomposition should
	 be identical.  In our case, this is replaced by the condition (Recall notation of Example \ref{scaling})
	 that ${\mathfrak R}\compcirc\sigma _{\lambda}(X,d_X) =
	 s_{\lambda}\compcirc{\mathfrak R}(X,d_X)$, which is
	 trivially satisfied.

	 \item Kleinberg's {second condition}, {\em richness}, asserts that
	 any partition of a dataset can be obtained as the result of
	 the given clustering scheme for some metric on the
	 dataset. In our context, partitions are replaced by
	 persistent sets. Assume that there exist $t\in \R_+$
	 s.t. $\theta_X(t)$ is the single block partition, i.e., impose
	 that the persistent set is a dendrogram (cf. Definition
	 \ref{def:persistent-sets}).  In this case, it is easy to
	 check that any such persistent set can be obtained as ${\mathcal
	 R}^{gen}$ evaluated for some (pseudo)metric on some
	 dataset. Indeed,\footnote{We only prove triangle inequality.}
	 let $(X,\theta_X)\in{ob}(\underline{\mathcal P})$.  Let
	 $\epsilon_1,\ldots,\epsilon_k$ be the (finitely many)
	 transition/discontinuity points of $\theta_X$. For $x,x'\in X$
	 define $d_X(x,x')=\min\{\epsilon_i\}$ s.t. $x,x'$ belong to
	 same block of $\theta_X(\epsilon_i)$.  
	
	 This is a pseudo metric on $X$. Indeed, pick points $x,x'$
	 and $x''$ in $X$. Let $\epsilon_1$ and $\epsilon_2$ be
	 minimal s.t. $x,x'$ belong to the same block of
	 $\theta_X(\epsilon_1)$ and $x',x''$ belong to the same block of
	 $\theta_X(\epsilon_2)$. Let
	 $\epsilon_{12}:=\max(\epsilon_1,\epsilon_2)$. Since
	 $(X,\theta_X)$ is a persistent set (Definition
	 \ref{def:persistent-sets}), $\theta_X(\epsilon_{12})$ must have
	 a block $\mathcal B$ s.t. $x,x'$ and $x''$ all lie in
	 $\mathcal B$. Hence $d_X(x,x'')\leq \epsilon_{12}\leq
	 \epsilon_1+\epsilon_2=d_X(x,x')+d_X(x',x'')$.

	 \item Finally, Kleinberg's {third condition},
	 \emph{consistency}, could be viewed as a rudimentary example
	 of functoriality. His morphisms are similar to the ones in
	 $\underline{\mathcal M}^{gen}$.
\end{itemize}

\subsection{Functoriality over $\mmon$}

In this section, we illustrate how relaxing the functoriality permits
more clustering algorithms.  In other words, we will restrict ourselves
to $\mmon$ which is smaller (less stringent) than
$\mgen$ but larger (more stringent) than
$\miso$. 

Fir a finite metric space $(X,d_X)$ and $r\geq 0$. For $x \in X$, let $[x]_r$ be the equivalence class of $x$ under the
equivalence relation $\sim _r$, and define $c(x) = \big|[x]_r\big|$.  For any positive
integer $m$, we now define $X_m \subseteq X$ by $X_m = \{ x \in X;\ 
c(x) \geq m\} $.  We note that for any morphism $f: X \rightarrow Y$
in $\mmon$, we find that $f(X_m) \subseteq Y_m$.
This property clearly does not hold for more general morphisms.  For
every $r$, we can now define a new equivalence relation $\sim ^m_r $
on $X$, which refines $\sim _r$, by requiring that each equivalence
class of $\sim _r$ which has cardinality $\geq m$ is an equivalence
class of $\sim _r^m$, and that for any $x $ for which $c(x) < m$, $x$
defines a singleton equivalence class in $\sim ^m_r$. We now obtain a
new persistent set $(X,\theta_X^m )$, where $\theta_X ^m (r)$ will denote
the partition associated to the equivalence relation $\sim ^m_r$.  It
is readily checked that $X \rightarrow (X, \theta_X^m )$ is functorial
on $\mmon$.  

\begin{definition}
For each $m\in\N$ let $\hfunc^{m}:\mmon\rightarrow \catp$ be the functor which to each finite metric space $(X,d_X)$ assigns the persistent set $(X,\theta^m_X)$ where for each $r\geq 0$, $\theta_X^m(r)$ is the partition of $X$ into equivalence classes of $\sim_r^m$.
\end{definition} 

This scheme could be motivated by
the intuition that one does not regard clusters of small cardinality
as significant, and therefore makes points lying in small clusters
into singletons, where one can then remove them as representing
``outliers'' or ``noise''. 

Another construction arises from similar consideration to those that gave rise to the standard clustering algorithms $\cfunc^{\Delta_m(\delta)}$ in Section \ref{section:mmon-density}, cf. Corollary \ref{coro:change-metric}.

\begin{definition}
For each $m\in\N$ define the functor $\rfunc^{\Delta_m}:\mmon\rightarrow\catp$ given by 
$$\rfunc^{\Delta_m} = \rfunc\compcirc \tfunc^m,$$
where we have written $\tfunc^m = \tfunc^\Omega$ for $\Omega= \{\Delta_m(\delta),\ \delta\geq 0\}$, cf. (\ref{eq:Tfunc}).
\end{definition}

That this definition produces a functor follows directly from the fact that $\rfunc$ and $\tfunc^m$ are functorial.

The resulting construction is in the same spirit as DBSCAN \cite{dbscan}.
\begin{figure}[htb]
 	\centering
 	\includegraphics[width=1\linewidth]{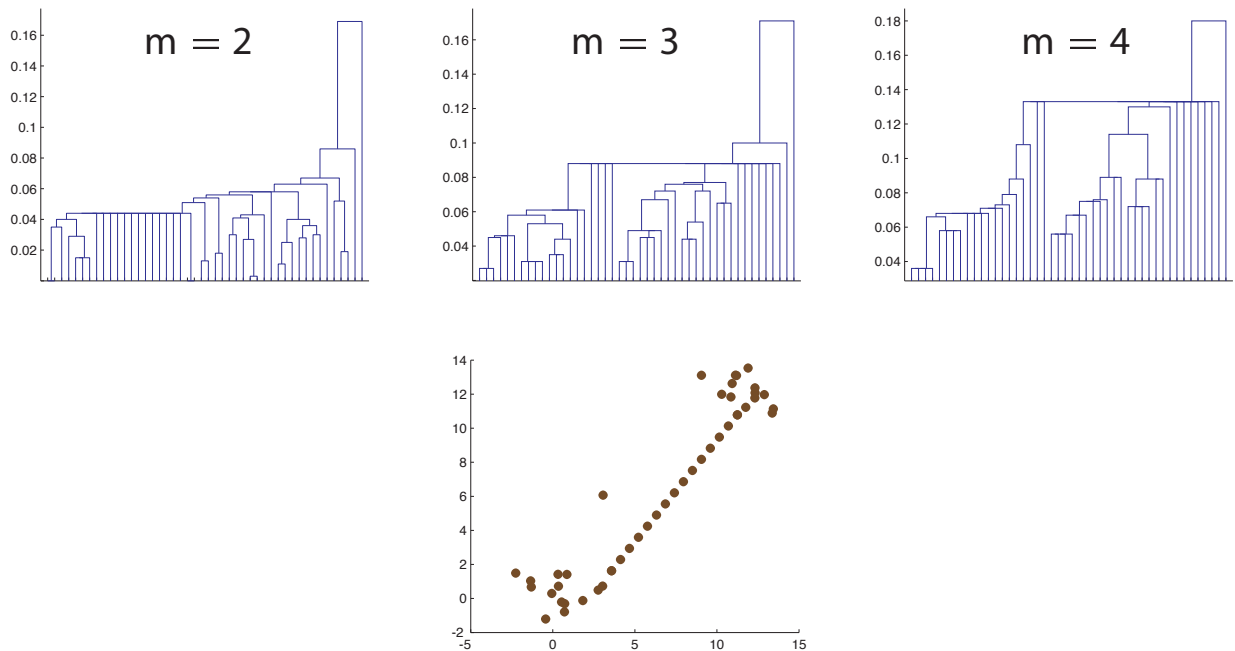}
 	\caption{Dendrograms resulting from applying $\rfunc^{\Delta_m}$ to a randomly generated dataset consisting of 50 points in the plane. The dataset consists of two clusters of points which are joined by a thin chain of points with some additional outliers. On the top row, from left to right we show the dendrograms corresponding to $m=2,3$ and $4$. Note how increasing $m$ produces dendrograms that exhibit more clearly clustered subsets.}  \label{fig:cDeltam}
 \end{figure}

\begin{remark} Notice first that $\rfunc^{\Delta_2}=\rfunc$. Also, 
for each $m\in\N$, the hierarchical clustering functor $\rfunc^{\Delta_m}$ is related to the standard clustering functor $\cfunc^{\Delta_m(\delta)}$ in the following manner. Fix $\delta\geq 0$ and write $\cfunc^{\Delta_m(\delta)}(X,d_X)=(X,P_X)$, then, $P_X=\theta_X^m(\delta).$
\end{remark}

An example of dendrograms obtained with these methods is shown in Figure \ref{fig:cDeltam}.



\section{Discussion} \label{various} 
The methods of this paper permit natural extensions to other situations,
such as  clustering of graphs and networks with, for example, the notion of clique clustering
\cite{hungarian-clique} fitting naturally into our context. It appears likely that from the point of view described
here, it will in many cases be possible, given a collection of
constraints on a clustering functor, to determine the {\em universal }
one satisfying the constraints.  One could therefore use sets of
constraints as the definition of clustering functors.

In addition to
producing classifications of clustering functors, functoriality is  a very
desirable property when clustering is used as an ingredient in
computational topology.  Functorial clustering schemes produce diagrams of
sets of varying shapes which can be used to construct simplicial complexes
which in turn reflect the topology of  connected components \cite{carlsson-09}.  Functorial schemes also can be used to construct ``zig-zag diagrams"
\cite{zigzag} which reflect the stability of
clusters produced over a family of independent samples.

We believe that the conceptual framework presented here can be a
useful tool in reasoning about clustering algorithms.  We have also
shown that clustering methods which have some degree of functoriality
admit the possibility of certain kind of qualitative geometric
analysis of datasets which can be quite valuable.  The general idea
that the morphisms between mathematical objects (together with the
notion of functoriality) are critical in many situations is
well-established in many areas of mathematics, and we would argue that
it is valuable in this statistical situation as well.

\end{document}